\documentclass[Afour,sageh,times]{sagej}
\usepackage{moreverb,url}
\usepackage{natbib}
\usepackage[colorlinks,bookmarksopen,bookmarksnumbered,citecolor=red,urlcolor=red]{hyperref}
\usepackage{amsmath}
\usepackage{amssymb}
\usepackage{mathtools}
\usepackage{balance}
\usepackage{graphics}
\usepackage{psfrag}
\usepackage{subfigure}

\newcommand\BibTeX{{\rmfamily B\kern-.05em \textsc{i\kern-.025em b}\kern-.08em
T\kern-.1667em\lower.7ex\hbox{E}\kern-.125emX}}

\usepackage{latexsym}
\newtheorem{proposition}{Proposition}
\newtheorem{remark}{Remark}
\newtheorem{definition}{Definition}
\newtheorem{lemma}{Lemma}

\usepackage{amsmath}
\usepackage{xspace}
\newcommand{\bpc}{binormalized Pl\"ucker coordinates\xspace}
\newcommand{\npc}{normalized Pl\"ucker coordinates\xspace}
\newcommand{\dd}[0]{\mbox{2{\sc d}}\xspace}
\newcommand{\ddd}[0]{\mbox{3{\sc d}}\xspace}

\newcommand{\matx}[1]{\mathbf{\uppercase{#1}}}

\newcommand{\vect}[1]{\mathbf{\lowercase{#1}}}
\newcommand{\unit}[1]{\vect{\underline{#1}}}
\newcommand{\Pnt}[1]{\mathnormal{\uppercase{#1}}}
\newcommand{\pnt}[1]{{\lowercase{#1}}}
\newcommand{\drte}[1]{{\lowercase{#1}}}
\newcommand{\Drte}[1]{{\uppercase{#1}}}

\newcommand{\id}[1]{\matx{I}_{#1}}
\newcommand{\zero}[2]{\matx{0}_{#1 \times #2}}

\newcommand{\hn}[0]{\unit{h}}
\newcommand{\un}[0]{\unit{u}}
\newcommand{\Vlin}[0]{\mathbf{V}}
\newcommand{\Vrot}[0]{\boldsymbol{\Omega}}

\renewcommand{\l}{\ell}
\renewcommand{\L}{\mathcal{L}}

\begin{document}
\runninghead{Andreff, Horaud and Espiau}
\title{Visual Servoing from Lines}
%\thanks{This work was supported by the European 
%  Community through the Esprit-IV reactive LTR project number 26247
%(VIGOR).
%During this work, Nicolas Andreff was staying at INRIA
%Rh\^one-Alpes as a Ph.D. student.}
\author{Nicolas Andreff\affilnum{1},  Bernard Espiau\affilnum{2} and Radu Horaud\affilnum{2}}
\affiliation{\affilnum{1}Institut Fran\c{c}ais de M\'ecanique Avanc\'ee,
BP 265, 63175 Aubi\`ere Cedex, France.
\affilnum{2}INRIA Rh\^one-Alpes and GRAVIR-IMAG,
655, av. de l'Europe, 38330 Montbonnot Saint Martin, France.}

\begin{abstract}
In this work, we present a fundamentally new approach to visual
servoing using lines. It is based on a theoretical and geometrical
study of the main line representations which allowed us to define a
new representation, the so-called binormalized Pl\"ucker
coordinates. They are particularly well suited to visual
servoing. Indeed, they allow the definition of a proper image line
alignment concept. The control law which realizes such an alignment has
moreover several properties: partial decoupling between rotation and
translation, analytical inversion of the motion equations and global
asymptotic stability conditions. This control law was validated both
in simulation and experimentally in the specific case of an orthogonal
trihedron.
\end{abstract}

\maketitle

%\section{Introduction}
\section{Motivation}
%\input{intro}
%\usepackage{palatino}
%\fontfamily{ppl}\selectfont
In this section we set the contribution of the paper in the 
general framework of visual servoing. After a brief analysis of the
evolution of the domain, we emphasize the need for considering
more systematically geometrical issues  as a way of improving the control.
\subsection{A Short Critical History of Visual Servoing}
The story of visual servoing goes back to  the seventies. At this time, the concept  of  robot began to to be commonly accepted  as the one of  a machine liable to interact (sense and act)
with its external environment with some degree of autonomy. Among the sensors
allowing to deliver the required measurements, cameras have immediatly
fascinated researchers for two reasons: the information which could potentially
be extracted from a set of images seemed almost without limits; the capacity
of providing a robot with vision fonctionnalities brought it nearer from human
skills.  Owing to this dramatic amount of information provided by cameras, the initial ambition in robotics was generally to {\em understand}
the world,  in order to decide, through symbolic approaches, how to act.  The
idea of using cameras inside a control loop came later from a rather
different point of view: by selecting adequate ''signals'' in an image, a camera becomes a sensor 
associated with a dynamical system, providing at every time it is needed with a kind of
error measurement with respect to a goal value. These two approaches grew
separatly from the  eighties: the first was linked to the computer vision domain, mainly leaded
by computer scientists and applied mathematicians, while the second gave rise
to the so-called {\em visual servoing} approach, investigated by people from
the automatic control area\footnote{a pioneering work is \citep{weiss87}.}. Formally, visual servoing is the art of designing
a control loop based on camera measurements, in the following way: 
\begin{itemize}
\item the control is the input $u$ of a system producing a {\em relative} euclidean motion (in $SE(3)$), between one or several camera(s)
and an observed target. The considered order of the system (1st or 2nd) and the nature of the
input space ($T \in  TSE(3))$, velocities or accelerations in a parametrized configuration
space $\{q\}$, actuator torques) depend on the type of used robot, on the assumptions made
about the overall dynamics and on the possible existence of internal loops;
\item the output space  is the one in which the task to achieve, $y^*$  is specified, {\em
and} in which the actual measurement $y$ is computed from image features, denoted
as $s$, which can be in addition a function of time. It should be emphasized
that whatever this final variable is (geometric figures in a image, estimated pose,
etc...), the specifity of the domain is that, somewhere in the process, a 
projective transformation is involved.
\end{itemize}
The domain is now rich of some hundreds of references and our goal is not to produce here an exhaustive survey (one of the first reviews on the subject is \citep{corke94} and a tutorial may be found in \citep{hutchinson96}). We will only focus
at the end of this section on the matter which is central to this paper,
illustrated with a few references. For the moment, let us  briefly evoke the main paths followed by the researchers this two last decades. They belong to three main classes: defining the output function, finding its 
variation and synthetizing the control. 
\begin{enumerate}
\item Task specification was one of the earliest addressed problem and we can nevertheless
say that it is still object of attention, a recent attempt to formalize it properly being found
in \citep{hespanha99}. The basic concern lies in the fact that there are many
ways of defining $y$ from $s$, with very different properties: the
trajectories in the various involved spaces can have extremely different shapes
even if the convergence is ensured; the existence and the
nature of possible singularities strongly depends on the choice of the output
function; the positionning accuracy is linked to the sensitivity of $y$ with respect to $u$. This is at the origin of numerous works, that are usually classified in two basic  categories: 2D ({\em image-based} - working directly on $s$) or 3D ({\em euclidean position
-based} - requiring reconstruction or pose estimation). Roughly, the first one is supposed to be less
sensitive to calibration uncertainties, the counterpart being that trajectories
of the moving system in the 3D space may be unpredictible. The second involves
more calculations and is less robust, but trajectories can be specified in the
euclidean space such that, for example, obstacles can be avoided. However, in that
approach, the relevant image features may leave the vision field. To overcome
these difficulties, researchers have imagined several methods, among them:
\begin{itemize}
\item to define optimal image features for a given application or find adequate representations \citep{fermuller98};
\item to specify goal image {\em trajectories} $s^*(t)$ instead of goal
positions \citep{martinet97}. The idea is based on the fact that the transient trajectories are
not predictible since the initial errors are large. Following a feasible
output trajectory with small errors is a way of being sure to stay in a true
regulation framework with better properties;
\item to use dynamical information \citep{cretual97a,chaumette94};
\item to combine specifications in 2D and 3D \citep{malis99a} or other spaces \citep{deguchi96}. We will see that the approach we propose fall in this category.
\end{itemize}

\begin{figure}[t!]
\centerline{\includegraphics[width=0.8\columnwidth]{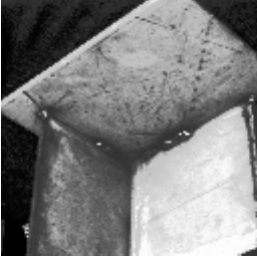}}
\caption{\label{photo_triedre} Detail of a part to be welded.}
\end{figure}

\item
A second  issue in the area was soon recognized as of critical importance:
the study of the variation of $y$. If we denote as $\bar{r}$ an element of $SE(3)$, and recalling that $q$ is the parametrization
of the space of actuation, we have 
\begin{equation}
\frac{\partial y}{\partial q} =\frac{\partial y}{\partial s }\frac{\partial s}{\partial \bar{r}} \frac{\partial \bar{r} }{\partial q}
\end{equation}
and the usual question is: what do we know about this expression? Note that the more relevant question,  what do we {\em really need} to know, is in fact related to robustness of the control.
	The last term refers to the physical system and its
knowledge involves robot calibration aspects. It should
be ignored if the control is defined only as the desired velocity screw of the
camera frame. The first term is supposed to involve analytic transformations
from the image space to the output space, which are user-defined and in
principle liable to be properly specified. The main concern lies in the
expression $\frac{\partial s}{\partial \bar{r}}$ called feature Jacobian,
interaction matrix or image jacobian. Since this matrix, denoted as $L$, results
from the differentiation of a mapping from $SE(3)$ to some subset of $R^n$, each
of its columns is an element of the cotangent space of $SE(3)$, in other terms a 
screw. It can therefore be systematically expressed in various frames using an
adjoint operator. As shown in \citep{espiau92}, where a general method of computation was proposed, it is also in general an explicit function of image measurements, camera calibration parameters and euclidean features. Since it is rarely possible to
capture all the theoretically involved variables, a great effort has been put for finding approximate expressions, $\hat{L}$: off-line estimation by learning the
matrix {\em entries} in suitable regions of the workspace \citep{hosoda95,jagersand96,deguchi96,sutanto98}; on-line identification
of some constant {\em parameters} related to the calibration \citep{kinoshita94}; use of the value
of $L$ computed at the goal position, etc... Besides, the euclidean variables
are set to ``reasonnable'' estimated values.
\item
The third research direction is closely related to the previous one. Indeed, when the
system is controlled through a feedback loop, the positioning accuracy mainly depends
on the capacity of the proportional gain to compensate for sources of steady-state
errors, like unmodelled potential-based terms, friction forces or unknown autonomous
feature motion. Besides, the effect of the
more or less good knowledge of $L$ concerns the transient stage: rate of convergence
and decoupling. In fact, intuitively speaking, stability is ensured if $\hat{L}$
is such that the generated motion has a component in the direction of the goal. This
refers to {\em robustness} issues, which were implicitly the main concerns of 
researchers from the automatic control side in visual servoing\footnote{ It is
also worth mentioning that the discrete-time aspects issued from the inherent
latency of the couple $\{$image acquisition - data processings$\}$ were also
an important and necessary topic of investigation  \citep{corke96}, while out of the scope of the
present paper.}: study of the sensitivity to calibration errors \citep{espiau93,miyazaki91},
 H-infinity control,
adaptive control, besides more classical approaches like LQ optimal synthesis \citep{hashimoto96}.
\end{enumerate}

Finally, up to the nineties, and but a few exceptions,
 works in visual servoing were mainly characterized by:
\begin{enumerate}
\item the use of a single camera linked to the controlled moving system;
\item an approach of the problem focused on automatic control
aspects and based on simple task definition ;
\item a system calibration assumed to be good enough to avoid disturbing the control too much;
\item the absence of a geometrical analysis deep enough to exploit the specificity of
the domain.
\end{enumerate}

This point of view culminated with the 1992 synthesis paper \citep{espiau92}. 
In the meantime, very relevant progresses in understanding the geometry of vision
began to inseminate the community. The use of projective and affine geometry began to 
appear rich of potentialities in visual servoing, mainly because it allows to
relax in some sense 
the problem of calibration and to obtain easily non-metric or quasi-metric
information very useful in all the three previously described issues.

Since 1993,
more and more papers began to incorporate these new approaches  (see for example
\citep{hollinghurst94,grosso95,joshi96,hosoda95,hager95,maru93,soatto94,hager97a,horaud98}
for stereovision).

\subsection{Geometry as a Central Issue in Visual Servoing}
But we can find also in the literature  recent works where a single camera is used, either
in an eye-in-hand approach \citep{yoshimi95,spratling96,cross96,colombo95,malis99a} or
as a global scene sensor \citep{ghosh96}. Affine geometry is used in \citep{cross96,colombo95}
and  \citep{spratling96} considers a weak perspective model. In \citep{malis99a}, a partial
pose computation allows the authors to propose a control scheme mixing image
features and quasi-geometric issues. It should be noticed that, like in the case of
stereovision, sensitivity to calibration uncertainties is here also a major concern. Many of the
cited references claim their robustness, obtained either by using a few specific motions
for self-calibration \citep{spratling96,cross96} or intrinsically, in general
implicitly thanks to the use of some geometric properties
\citep{malis99a,yoshimi95}). 
Clearly,
we can understand this need for considering explicitely  geometrical aspects from the
following facts:
\begin{enumerate}
\item when using camera(s), the process for obtaining the data used in the feedback 
involves always a projective transformation which confers particular characteristics
to the models;
\item the efficiency of control loops is often better if some information about
the euclidean space can be recovered from image data, which requires
dedicated geometrical methods like reconstruction or pose estimation;
\item as soon as we use structures richest than simple points, features in the image
 result from 3D entities which  belong to particular
manifolds. The choice of the parametrization of these manifolds (and of the
projected structures) strongly influences the design of the control law. Among
the expected benefits are the better decoupling, the easier design and
stability analysis of the control.
\end{enumerate}

\subsection{Motivation of this Work}
The present paper is mainly concerned with this last aspect, and more precisely
with the use of lines. At the origin of the problem is the fact that when
dealing with industrial parts, the extraction and the tracking of lines is natural
and, now, feasible in real time thanks to new efficient algorithms from
computer vision \citep{lowe91a,marchand98a,tonko00,Armstrong95,drummond00}. Nevertheless,
only a few works have really explored the use of lines in 
visual servoing: \citep{rives87b,espiau92,motyl92a,debain96a,hager97a}
are examples we find in robotics.  
 Tracking of lines was also considered for the automatic road following of autonomous vehicles,
like in \citep{lutzeler98}, but without deriving generic approaches.

 As explained in section \ref{results}, the experimental
setup we considered in the framework of the european project VIGOR was a ship part
provided by the Odense Steel Shipyard in Denmark for which it was required
to accurately position a welding tool. Several intersecting planes were present in
the part, making the extraction of lines easier than that of points.  The question which rapidly 
occured was then: what representation for these lines?

Coming back to geometrical aspects, it should be recalled that the 3D Euclidean lines form
a 4-dimensional manifold the properties of which are described by the Grassman's geometry. 
One of the most noticeable property  is the absence of a natural (Riemanian) distance. Since
such a distance is somewhere needed for defining a feedback control, the choice of 
the parametrisation is crucial. Among the classical ones, the Pl\"{u}cker coordinates
look particularly interesting since  they come from the embedding of the line
manifold in a projective space in which things are linear. More, recalling that 
the camera transformation is also projective, we understand intuitively that some
nice properties could be expected. In this framework, the main  original contributions
of the paper are the following. First, we propose and discuss a restriction on classical
Pl\"{u}cker coordinates, which looks very pertinent in visual servoing. Secondly, we
design a control law on this submanifold  and prove its stability. Then we apply the method
to an orthogonal trihedron, and extend it to a control scheme including  a laser beam.

%\section{Notation}
\begin{table}[ht]
\footnotesize
\fbox{\begin{minipage}{\linewidth}
\begin{itemize}
\item Vectors are denoted with lower case upright bold letters
(e.g.$\vect{u}$). Unit vectors are denoted as underlined vectors
(e.g. $\unit{u}$). 
\item Points in the \ddd space are written with upper case bold
letters ($\Pnt{P}$) while points in the image are written with
lower-case bold characters ($\pnt{p}$). 
\item \ddd Lines are noted with calligraphic upper case letters (e.g. $\L$)
and image lines with calligraphic lower case letters (e.g. $\l$).
%\item Matrices are denoted with upper-case bold-face capital letters
%(e.g. $\matx{A}$). 
\item Angular velocity is written $\Vrot$ and linear velocity $\Vlin$.
\item The scalar product is noted with the transpose formalism
(e.g. $\vect{a}^T \vect{b}$) and the cross product of $\vect{a}$ by $\vect{b}$ is written
$\vect{a} \times \vect{b}$.
\item The notation $^*$ (e.g. $\unit{u}^*$) denotes a desired value and the notation
$(t\!=\!0)$ (e.g. $\unit{u}(t\!=\!0)$) represents an initial value. 
\end{itemize}
\end{minipage}
}
\caption{Notations used in this paper}
\end{table}

\section{Line parameterization}
\subsection{Binormalized Pl\"ucker coordinates}
\subsubsection{Recalls on Pl\"ucker coordinates}

In \citep{plucker65}, Pl\"ucker derived 8 kinds of coordinates for a line
in space  (or \ddd line) from  the equations, in  the projective space
$\mathcal{P}^3$, of two intersecting planes.
Among these 8 kinds of coordinates, the so-called Pl\"ucker
coordinates are defined as follows 
\citep{bottema79,faugeras93a,semple52a}.  

Consider any two points, in the projective space,
$\Pnt{P}=(x_1,x_2,x_3,x_4)$ and 
$\Pnt{P}'=(x'_1,x'_2,x'_3,x'_4)$ a the \ddd line $(\L)$. The latter
is uniquely defined by these two points. However, this representation
is far from being minimal since it makes use of 8 values to
parameterize the 4-dimensional manifold of lines. 

To reduce the amount of parameters, one considers the 12 $(2 \times 2)$
minors of the matrix $\left(
\begin{smallmatrix}
x_1&x_2&x_3&x_4\\
x_1'&x_2'&x_3'&x_4'
\end{smallmatrix} \right)$. Six of them can be chosen
independent. Moreover, they are unique up to a single scale factor for 
every choice of the two points $\Pnt{P}$ and
$\Pnt{P}'$\cite[p24]{faugeras93a}. 

Thus, the projective Pl\"ucker coordinates are defined by:
%\begin{equation}
%\L=\begin{pmatrix} p_{41}& p_{42}& p_{43}& p_{23}& p_{31}&
%p_{12}\end{pmatrix}^T \quad with \quad
% p_{ij} = x_i x'_j - x'_i x_j
%\label{coord_pluck_proj}
%\end{equation}
%which we rewrite as:
\begin{equation}
\label{coord_pluck_proj}
\L=\begin{pmatrix}
\vect{u}\\ \vect{v} 
\end{pmatrix} 
\quad with \quad
\begin{matrix}
\vect{u}=(p_{41}, p_{42}, p_{43})^T\\
\vect{v}=(p_{23}, p_{31}, p_{12})^T
\end{matrix}
\quad p_{ij} = x_i x'_j - x'_i x_j
\end{equation}
Notice that this is a bivector as McCarthy defined
it\citep{maccarthy90a} to represent kinematic screws. Consequently,
as for $se(3)$\cite[p.~427]{murray94a}, there does not exist a
bi-invariant Riemannian metric on the manifold of \ddd lines.

One recognizes in $\vect{u}$ the first 3 coefficients 
of the 4-vector $x_4 \Pnt{P}' - x'_4 \Pnt{P}$, the fourth one being
zero. It is a classical result in projective geometry that this
4-vector is the point on $(\L)$ which lies at the 
infinite, and thus represents its direction.
Hence, $\vect{u}$ corresponds to the direction of $(\L)$.

The 6 homogeneous coordinates in~(\ref{coord_pluck_proj}) define a
point, sometimes called the \emph{Pl\"ucker point}\citep{goodman97}, in
the 5 dimensional projective space. Consequently, 5 of these
coordinates are independent. Hence, one extra relation is needed to
turn down to the 4-dimensional manifold of lines. The usual extra relation
expresses that a line intersects itself.

Consider two lines $(\L_1)$ and $(\L_2)$ respectively
defined by the pairs of points $(\Pnt{P},\Pnt{P}')$ and
$(\Pnt{Q},\Pnt{Q}')$. $(\L_1)$ and $(\L_2)$ intersect if
$\Pnt{P}$, $\Pnt{P}'$, $\Pnt{Q}$ and $\Pnt{Q}'$ are coplanar:
\[
\begin{vmatrix}
x_1&x_2&x_3&x_4\\
x_1'&x_2'&x_3'&x_4'\\
y_1&y_2&y_3&y_4\\
y_1'&y_2'&y_3'&y_4'
\end{vmatrix}=0
\]
Expanding this determinant, called {\em Grassmannian}, one obtains:
\begin{equation}
\label{pq=0}
p_{41}q_{23}+p_{42}q_{31}+p_{43}q_{12}+p_{23}q_{41}+p_{31}q_{42}+p_{12}q_{43} 
= 0
\end{equation}
or, with our notations:
\begin{equation}
\vect{u}_1^T\vect{v}_2 +\vect{u}_2^T\vect{v}_3=0
\end{equation}
The missing relation is thus:
\begin{equation}
\label{Klein}
\vect{u}^T\vect{v}=0
\end{equation}
which defines a 4-dimensional manifold of
$\mathcal{P}^5$, called the {\em Klein quadric}\citep{pottman98}
or the {\em Pl\"ucker hypersurface}\citep{chazelle96a}.
Thus, any point on this quadric corresponds to a projective line in
the physical world and reciprocally. This correspondence is unique if
the lines and  the projective space are given an orientation, as
suggested by Stolfi\citep{stolfi91a}. 

As for points, the upgrade from projective to Euclidean coordinates of
a  \ddd  line  is  obtained  by normalization.   Thus,  the  Euclidean
Pl\"ucker  coordinates  of  a \ddd  line  can  be  seen as  {\em  \npc
}\citep{pottman98}.   The  normalization which  is  commonly  chosen is  that
$\vect{u}$ should be a unit vector. Therefore, the \npc of a \ddd line
defined     by      the     (projective)     Pl\"ucker     coordinates
$(\vect{u},\vect{v})$ are:
\begin{align}
\L&=(\unit{u},\vect{h})\\ 
\unit{u}^T\vect{h}&=0\\
\unit{u}^T\unit{u}&=1
\end{align}
where               $\unit{u}=\vect{u}/\|\vect{u}\|$               and
$\vect{h}=\vect{v}/\|\vect{u}\|$.  It can  easily be shown that, given
any \ddd point $\Pnt{P}$ on $(\L)$, one has:
\begin{equation}
\label{h=Pxu}
\vect{h}=\Pnt{P}\times \unit{u}
\end{equation}
Consequently, the following equivalence holds:
\begin{equation}
\vect{h}=0 \Leftrightarrow (\L) \ passes\ through\ the\ origin
\end{equation}

\paragraph{Working assumption --} Since the case where a line passes
through the center of projection is degenerate in computer vision, we will
assume in the remainder of this paper that it never occurs.

\begin{figure}[ht]
\psfrag{h}{$\vect{h}=h \hn$}
\psfrag{u}{$\unit{u}$}
\psfrag{uxn}{$\unit{u} \times \hn$}
\psfrag{v}{$\vect{v}$}
\psfrag{w}{$\vect{w}$}
\psfrag{n}{$\hn$}
\psfrag{d}{$\drte{d}$}
\psfrag{D}{$\Drte{D}$}
\psfrag{P}{$\Pnt{H}$}
\psfrag{p}{$\pnt{p}$}
\psfrag{O}{$\Pnt{O}$}
\centerline{\resizebox{!}{5cm}{\includegraphics{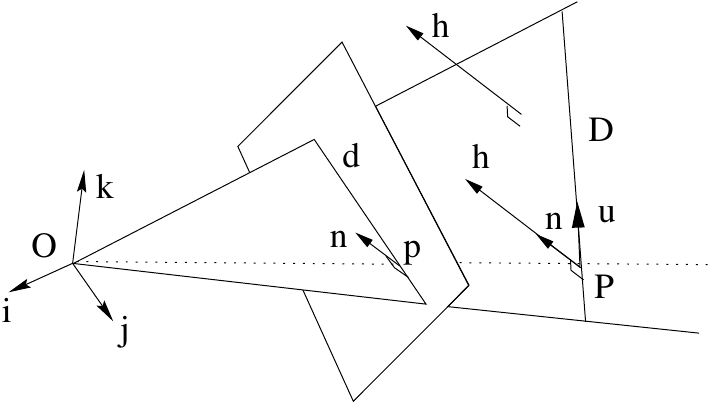}}}
\caption{\label{interpretation_plucker}Geometric interpretation
of the \npc of a line.}
\end{figure}

\paragraph{Geometric interpretation} In the computer vision community,
the geometrical interpretation of 
the \npc is the following (see  Figure~\ref{interpretation_plucker}):
$(\L)$  is the line  oriented by the unit  vector $\un$
and passing through  $\Pnt{H}=\un\times\vect{h}$, which is the closest
point  to the  origin  on  $(\L)$\citep{navab}.  It  can also be noticed
from~(\ref{h=Pxu}) that  $\vect{h}$ is orthogonal to  the plane passing
through  the  origin  and  containing $(\L)$.  This  plane  also
contains the projection $(\l)$  of $(\L)$ onto the image plane
and is referred to as the {\em interpretation plane} of $(\l)$.
As the interpretation plane passes through the origin, it is defined
by:
\begin{equation}
ax+by+cz=0
\end{equation} 
Considering any point in the interpretation plane, such that its $z$
coordinate is not zero, one can rewrite the previous equation as:
\begin{equation}
a \frac{x}{z}+b \frac{y}{z} + c =0
\end{equation}
where we recognize the equation of a line in the image plane. This
line is thus the intersection of the image plane and the
interpretation plane, hence~$(\l)$.

Therefore, ($\l$) is uniquely defined by the interpretation plane, hence
by its unit normal vector $\hn=\vect{h}/h$ where $h = \| \vect{h} \|$.
Consequently, the normalized Pl\"ucker coordinates of a line ($\L$) can
be expressed as:
\begin{equation}
\L = (\un,h \hn)
\end{equation}
with the following constraints:
\begin{align}
\un^T\hn &= 0\\
\un^T\un = \hn^T \hn &= 1
\end{align}
Since the projection ($\l$) of ($\L$) onto the image defines the
interpretation plane, it has for coordinates those of the normal
3-vector to this plane:
\begin{equation}
\l = \hn
\end{equation}

\subsubsection{Binormalized Pl\"ucker coordinates}
\label{bpc}

\paragraph{Line parameterization}

From the previous section, we are now ready to define the \bpc, as a
way to parameterize \ddd lines, and to give some properties of these coordinates.

\begin{definition}
The \bpc of the pencil of all the \ddd lines oriented by the unit
vector $\un$ and lying in the plane of unit normal vector $\hn$ are the couple $(\un,\hn)$
where
\begin{align}
\un^T\hn &= 0\\
\un^T\un = \hn^T \hn &= 1
\end{align}
\end{definition}

By extension, we will call \bpc of a \ddd line the \bpc of the pencil
it defines.
With this definition, we can trivially state the following
proposition:

\begin{proposition}[\ddd line representation]
A \ddd line ($\L$) is uniquely defined by its \bpc $(\un,\hn)$ and its depth
$h$:
\begin{equation}
\L = \begin{pmatrix} \hn\\ \un\\ h \end{pmatrix}
\end{equation}
\end{proposition}

Notice that, as we consider ($\L$) lying in a pencil of
parallel lines, $h$ is not only a distance but also becomes a depth.

\paragraph{Geometrical interpretation}
Our representation yields a sequential interpretation
(Figure~\ref{interpretation_seq}): $(\L)$ is the line in the
interpretation plane of $(\l)$ (defined by its normal vector $\hn$)
with orientation $\un$ and lying at depth $h$. This interpretation
corresponds to the usual sequence followed in image analysis to
determine a \ddd line from one or several images: extraction of 
$(\l)$ from the image, then computation of the orientation of $(\L)$,
and finally, computation of the \ddd position of $(\L)$.

Before going further let us notice that from the definition of the
\bpc, we can make the following two remarks:

\begin{remark}
\label{bon}
The triple $(\un,\hn,\un\times\hn)$
forms an orthonormal basis of the \ddd space.
\end{remark}

\begin{remark}
\label{h=|H|}
Let $\Pnt{H}$ be the closest point to the origin n a line
$\L=(\hn^T,\un^T,h)$. Then, $\Pnt{H} = h \un \times \hn$.
\end{remark}

Finally, representing lines with \bpc and depth allows to generate
some well-known geometrical objects from lines, by fixing one or
several components to a given value. For instance, the lines with a
given depth $h$ generates a sphere of radius $h$. Figure~\ref{objets}
gathers the other cases.

\begin{figure*}[t]
\psfrag{h}{$h$}
\psfrag{u}{$\unit{u}$}
\psfrag{hn}{$\hn$}
\psfrag{d}{$\l$}
\psfrag{D}{$\L$}
\psfrag{O}{$\Pnt{O}$}
\centerline{\includegraphics[width=0.8\textwidth]{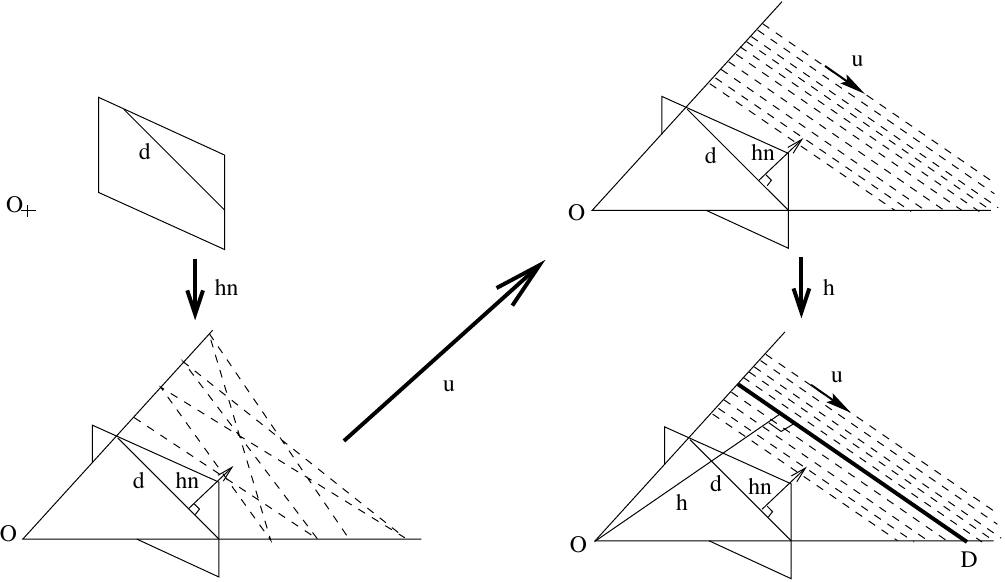}}
\caption{The binormalized Pl\"ucker coordinates have a sequential interpretation: from image measurements to the estimation of \ddd components.}
\label{interpretation_seq}
\end{figure*}

\begin{figure*}[h!]
\psfrag{Z}{$h$}
\psfrag{u}{$\unit{u}$}
\psfrag{h}{$\hn$}
\psfrag{d}{$\l$}
\psfrag{D}{$\L$}
\psfrag{O}{$\Pnt{O}$}
\centerline{\resizebox{\textwidth}{!}{\includegraphics{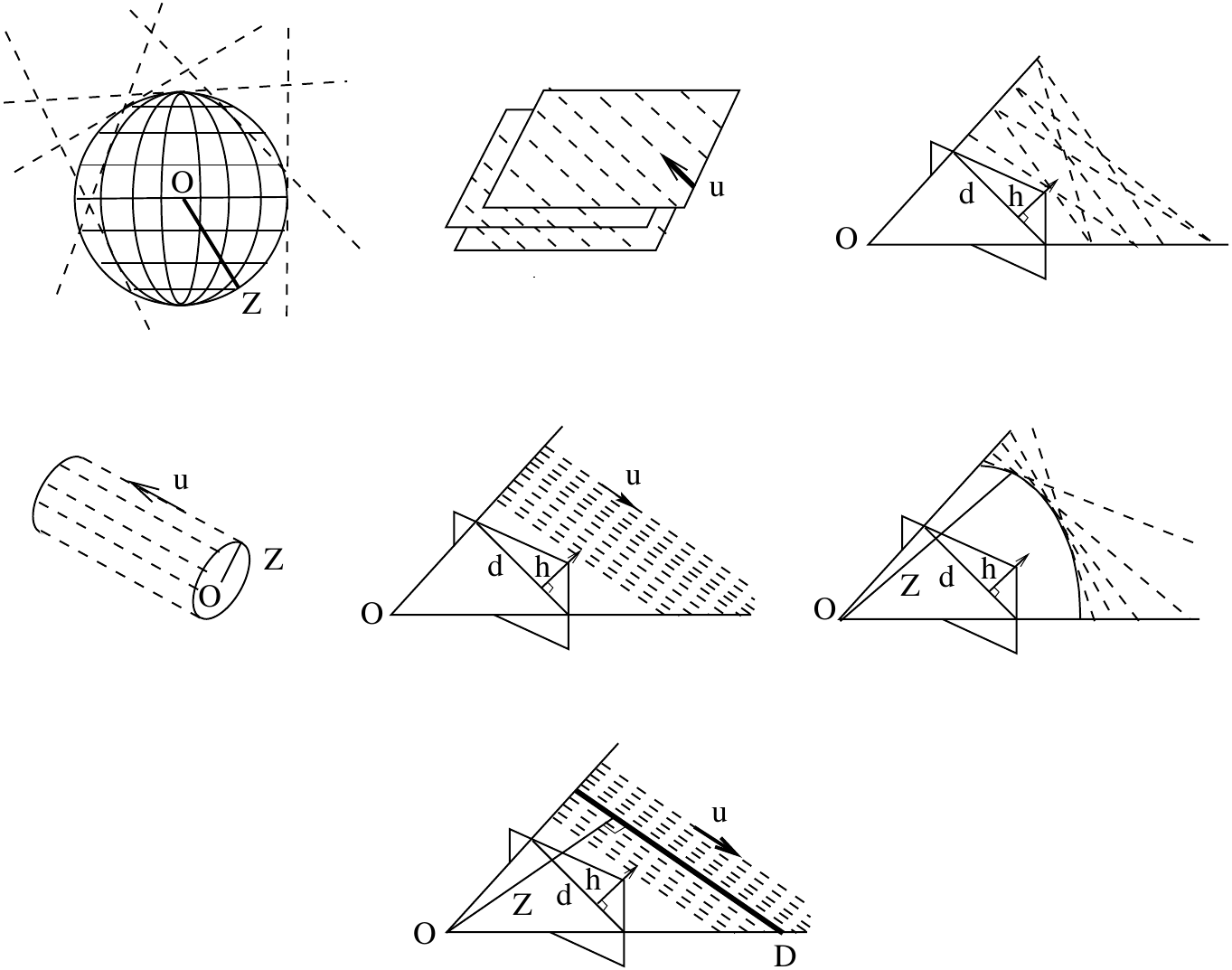}}}
\caption{Geometrical objects generated by lines with fixed components
of their \bpc}
\label{objets}
\end{figure*}

\subsubsection{Image alignment of a \ddd line} Among the previous cases, we are
interested in those where the image projection $\hn$ is given. Indeed, 
visual servoing consists in defining the goal position of a moving system by
an image, in our case, by the projection of a set of lines into the
image. On the opposite to the case where visual servoing is based on
points, here we do not have a simple answer to the question: ``What
are the configurations in \ddd space that give the same image~?''
As seen in Figure~\ref{objets}, there are four cases:
only $\hn$ is given; $\hn$ and $h$ are given; $\hn$ and $\un$ are
given; $\hn$, $\un$ and $h$ are given. The first case corresponds to
the requirement used by Chaumette 
{\em et al}\citep{chaumette90,espiau92a,motyl92a,debain96a}. The
second and fourth ones are not very interesting since they require the
estimation of the depth, always hard to accurately achieve.

On the opposite, the third case is highly
interesting. The \ddd line is constrained by its image projection and
its alignment on a given direction. Hence, we call this the {\em image
alignment} of a \ddd line. It corresponds to fixing the \bpc of the
\ddd line to a given value. As stated before, the only remaining
degree of freedom is depth. Hence, this brings us back to the case of
points, where the image of a point defines a 1-dimensional manifold in 
the \ddd space.

In addition to this elegant similarity with the case of points, the
image alignment matches the property of partial decoupling given
below, which allows the definition of a control law with global
convergence proof (see section~\ref{control}).

\subsubsection{Motion of a line vs. motion of the camera}

Here, we relate the instantaneous camera motion to the \ddd
line motion (expressed in our formalism) and recall the apparent
motion in the image of this line.

The instantaneous motion of a camera is defined by its velocity screw
$\tau=(V,\Omega)$, where $\Omega$ is the instantaneous
angular velocity and $V$ the instantaneous linear velocity
of a given point. They are usually expressed in the camera frame.

Let $\Drte{L}= (\hn{}^T,\unit{u}^T,h)^T$ be a \ddd line. Then its motion
is the vector $(\dot{\hn}{}^T,\dot{\unit{u}}{}^T,\dot{h})^T$. 
Rives {\em et al}~\citep{rives87b} recall the derivative
$(\dot{\unit{u}},\dot{\vect{h}})$ of the \npc $(\unit{u},\vect{h})$ of
a \ddd line~(\ref{du})--(\ref{dvecth}) 
\begin{align}
\label{du}
\dot{\unit{u}} &= \Omega \times \unit{u}\\
\label{dvecth}
\dot{\vect{h}} &= \Omega \times \vect{h} + V \times \unit{u}
\end{align}
The null space of this mapping is given by
\begin{align}
\label{ns1}
\Omega_{ns} &= \lambda \un\\
\label{ns2}
V_{ns} &= \mu \un + \lambda \vect{h}
\end{align}
where $\mu$ and $\lambda$ are real numbers.

From the definition of $h$, we have:
\begin{align}
\dot{h} &= \frac{d}{dt}(\sqrt{\vect{h}^T\vect{h}}) =
 \frac{1}{h}(\vect{h}^T \dot{\vect{h}}) = \hn^T\dot{\vect{h}}
\label{hnTh}
\end{align}
Then, inserting $\dot{\vect{h}}$ from (\ref{dvecth}) into (\ref{hnTh}) yields:
\begin{align}
\label{dh}
\dot{h}&=\hn^T (\Omega \times \vect{h} + V \times \unit{u})
= \Omega^T (\vect{h} \times \hn) + V^T (\unit{u} \times \hn)
\end{align}
Since $\vect{h} \times \hn=0$, we finally obtain:
\begin{equation}
\dot{h}=  V^T (\unit{u} \times \hn )
\end{equation}

Then, one can obtain $\dot{\hn}$, which Navab~\citep{navab} calls the {\em line
motion field equation}. Indeed, it can be 
interpreted as
the relation between the camera motion and the apparent motion of a \ddd line
in the image. It is obtained by:
\begin{equation}
\dot{\hn} = \frac{d}{dt}(\frac{\vect{h}}{h}) = 
\frac{1}{h}\dot{\vect{h}} - \frac{\vect{h}}{h^2} \dot{h}
\end{equation}
Then, insert $\dot{h}$ from
(\ref{dh}), $\dot{\vect{h}}$ from (\ref{dvecth}) and replace
$\vect{h}/h$ by $\hn$ whenever it is possible:  
\begin{equation}
\label{compare_hager}
\dot{\hn}=\frac{1}{h}(\id{3}- \hn \hn^T)
(\Omega \times \vect{h} + V \times \un)
\end{equation}
which simplifies into
\begin{equation}
\label{dhn}
\dot{\hn} = \Omega \times \hn - \frac{V^T \hn}{h}(\unit{u} \times \hn )
\end{equation}

%Let us recall the derivation of (\ref{dhn}) from
%(\ref{dvecth}). First, express $\dot{\hn}$ in function of $\dot{\vect{h}}$:
%\begin{align}
%\dot{\hn}&=\frac{d}{dt}(\frac{\vect{h}}{\sqrt{ \vect{h}^T\vect{h}}})\\
%&=\frac{ \sqrt{ \vect{h}^T\vect{h}}\ \dot{\vect{h}} - \vect{h}\ 
%\frac{d}{dt} (\sqrt{ \vect{h}^T\vect{h}}) }{\vect{h}^T\vect{h}}\\
%&=\frac{1}{\sqrt{ \vect{h}^T\vect{h}}}(\id{3}-\hn\hn^T)\dot{\vect{h}}
%\end{align}
%Then, insert $\dot{\vect{h}}$ from (\ref{dhn}) and replace $\sqrt{
%\vect{h}^T\vect{h}}$ by $h$ in the latter:
%\begin{equation}
%\label{compare_hager}
%\dot{\hn}=\frac{1}{h}(\id{3}-\hn\hn^T)(\Omega \times \vect{h} + V
%\times \un)
%\end{equation}
%which simplifies into:
%\begin{equation}
%\label{simplification}
%\dot{\hn}=\Omega \times \hn + \frac{1}{h}(\id{3}-\hn\hn^T)(V \times \un)
%\end{equation}
%since $\hn^T(\Omega \times \vect{h})=0$. Then, one obtains (\ref{dhn}) 
%using Proposition~\ref{bon}.

Finally, the Jacobian, or interaction matrix in the task function
approach\citep{samson90a,espiau92a}, associated to our
parameterization is the matrix in: 
\begin{equation}
\label{jac_hn}
\begin{pmatrix}
\dot{\hn}\\ \dot{\unit{u}}\\ \dot{h}
\end{pmatrix}
= \begin{pmatrix} -\frac{1}{h}(\unit{u} \times \hn) \hn{}^T &
-As(\hn) \\
 \zero{3}{3} &
-As(\unit{u})\\
(\unit{u} \times \hn)^T &
\zero{1}{3} \end{pmatrix} \begin{pmatrix} \Vlin\\ \Vrot \end{pmatrix}
\end{equation}
exhibiting a partial decoupling between rotation and translation.

\subsection{Discussion}

The parameterization we propose, based on \bpc, has some advantages,
from the control point of view, over the parameterizations found in
the litterature. Before we go on with the description of the drawbacks 
we found in other representations, let us recall these advantages:
\begin{itemize}
\item Our representation is continuous and covers the whole space,
except the case of lines going through the optical center.
\item It is coherent in the \ddd space and the image.
%\item It follows the sequence of image analysis treatments.
\item It is unique provided that some constraints, easy to satisfy, hold.
\item It exhibits a  partial decoupling between rotation and translation.
\end{itemize}

\subsubsection{In the 3D space}

Apart from the Pl\"ucker coordinates, one usually finds three other
representations of a \ddd line. All of them have their drawbacks.
\begin{enumerate}
\item A line is defined by a point and a direction.

The first drawback of this representation is that it requires to
select a point on the \ddd line. This is troublesome in the
visual guidance framework since there may not be any distinguishable
physical point on the line which is observed by the camera. Moreover,
it turns down to extracting a point in the image, which is what we
want to avoid by using lines.

Even though one could use points on the line, the choice of the point
is not unique. Consequently, there may be some discrepancy between the 
point chosen to define the desired position of the line and the point
chosen to define the current position of the same line. Even if the
same point were used, this would require some {\em matching} process.

Finally, this representation is deconnected from what can be directly
extracted from the image.

In fact, to make this representation point independent, or unique, or to
relate it to the image, is to come back to Pl\"ucker coordinates.
We remarked that Hager\citep{hager97a} represents lines with a point
and a direction, but comes to the Pl\"ucker coordinates as soon as
control is involved.

\item A line is the intersection of two planes.

The main drawback is that this representation is implicit since it
defines a line as the set of points verifying two plane equations:
\[
\left\{ \begin{matrix}
		a_1 x + b_1 y + c_1 z + d_1 = 0\\
		a_2 x + b_2 y + c_2 z + d_2 = 0
		\end{matrix}
		\right.
\]
Control is much easier with an explicit parameterization~!
Nevertheless, Chaumette\citep{chaumette90,espiau92a} used it, together
with the $(\rho,\theta)$ representation of an image line (see below),
but to the cost of motion equations with no specific structure. In particular,
rotation and translation are highly coupled.

\item The $(a,b,p,q)$ representation\citep{faugeras93a,navab93,plucker65}

This representation is a minimal variant of the previous one:
\[
 \left\{ \begin{matrix}
		x = a z + p\\
		y = b z + q
		\end{matrix}
		\right.
\]
It represents the lines that are neither parallel to the optical axis, 
nor parallel to the image plane. To take into account all the lines in 
the \ddd space, one has to make permutations over $x$, $y$ and
$z$. One thus obtains three overlapping but nevertheless discontinuous 
representations that gather all the lines in space. However, the
discontinuity of the representation would require switches in the
control law, that may have destabilizing effects or, at least lead to unacceptable transient behaviors.
\end{enumerate}

\subsubsection{In the image}

Consider a \ddd line ($\L$) with \bpc $(\hn^T,\un^T,h)$ and projecting
into the image as the line ($\l$). Recall that we represent ($\l$) by
the unit vector $\hn$. This corresponds to the normalization of the
3-vector containing the line equation coefficients ($a,b,c$) in the
\ddd space.

We give now the other usual representations of a line
and show why they are not
as well suited to the visual servoing problem as the one we chose.
\begin{enumerate}
\item The unnormalized triple $(a,b,c)$ of its equation coefficients.

Let us discard it immediately since it is not uniquely defined for a
given line, which is hard to handle from the control point of view.

\item The normalized triple $\vect{d}=(a,b,c)=(\cos \theta, \sin \theta, \rho)$.

This parameterization is used in visual servoing by Hager\citep{hager97a}.
It corresponds to the normalization of the line equation coefficients
on a subspace of dimension 2, i.e.
\begin{equation}
\vect{d}^T 
\matx{N}
\vect{d} = 1
\quad with \quad
\matx{N}=\begin{pmatrix}
1&0&0\\
0&1&0\\
0&0&0
\end{pmatrix}
\end{equation}
Hence, 
\begin{equation}
\vect{d}=\frac{\vect{h}}{\sqrt{\vect{h}^T\matx{N}\vect{h}}}
\end{equation}

%\begin{corrige}{A supprimer/developper ?}
%It is noticeable that this normalization operator is only
%semi-definite positive, while ours is definite positive.
%\end{corrige}

The first drawback of this parameterization is that $\vect{d}$ is not
a unit vector in the \ddd space since its norm equals
\mbox{$\sqrt{1+\rho^2}$}. Consequently, 
one does not have the opportunity to later build 
upon it an orthonormal basis as we did (Remark~\ref{bon}).
This results in an heterogeneity between image and space
parameterizations, especially as far as metric information is concerned.
Indeed, if one can, in some way, obtain the orthogonal distance $z$ of
the line to the origin, then, one has to perform an ``unnormalization''
in order to reconstruct the point of the line $\Pnt{H}$ which is
the closest to the origin:
\[
\Pnt{H}=\frac{z}{\|\vect{d}\|}\ \un \times \vect{d}
\]
 Moreover, the unnormalization factor is not constant but
depends on the projection of the line onto the image.
On the opposite, with our parameterization, we can immediately,
and independently from the image information,  obtain
$\Pnt{H}$ since, from Remark~\ref{h=|H|},
\[
\Pnt{H}=z\ \un\times\hn
\]

A second drawback is that the consequent motion line equation is more
complex than ours. 
Indeed, rewriting the Jacobian associated with this parameterization
with our notation, one obtains the following motion line equation
instead of~(\ref{compare_hager}):
\[
\dot{\vect{d}} = \frac{1}{d}(\id{3}-\vect{d}\vect{d}^T\matx{N})
(\Omega \times \vect{h} + V \times \un)
\]
where $d=\sqrt{\vect{h}^T\matx{N}\vect{h}}$. Thus, the presence of
$\matx{N}$ does not allow such a simplification of this expression
as from (\ref{compare_hager}) to (\ref{dhn}).
In particular, the term $\vect{d}^T\matx{N}(\Omega \times \vect{h})$
does not vanish, preventing from expressing the rotation effects
in the image as a simple cross-product.

\item The minimal representation $(\rho,\theta)$

This representation is the other representation used in visual
servoing\citep{chaumette90,espiau92a,motyl92a,debain96a}. It has the
advantage of being minimal. However, it is periodic. Thus, the angle
$\theta$ is defined either in $(-\pi,\pi]$ or in $[0,2\pi)$. 
This is troublesome in two ways.

Firstly, the choice of the definition interval of $\theta$ is not
generic. It depends indeed of  initial ($\theta_0$) and
desired ($\theta^*$) values
 of $\theta$. Take two examples
(Figure~\ref{theta}). When $\theta_0>0$ and
$\theta^*<0$ are close to 0, it is clear that one should chose $\theta
\in  (-\pi,\pi]$, otherwise it would travel almost the whole unit
circle. On the opposite, when $\theta_0$ and $\theta^*$ are close
to $\pi$, the choice is trivially $\theta \in [0,2\pi)$.

\begin{figure*}[t]
\psfrag{t0}{$\theta_0$}
\psfrag{t*}{$\theta^*$}
\centerline{\resizebox{!}{0.4\textwidth}{\includegraphics{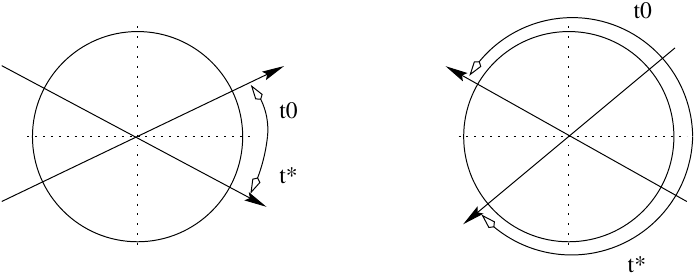}}}
\caption{Two different choices for the definition interval of $\theta$}
\label{theta}
\end{figure*}

This {\em ad hoc} choice does not prevent from the second drawback of
this representation: it restricts the variations of $\theta$.
Indeed, if the line in the image rotates too much, then $\theta$ may
reach one of its bounds or, even worse, go beyond it. Then, it should
have to switch to the other bound. This, in turn, would generate a
discontinuity in the control law. 

\end{enumerate}

\section{Control}
\label{control}
\subsection{Introduction}

Let us consider a fixed reference frame $F_r$ and a moving object with associated frame $F_0$ (here the 
camera frame). We
define a goal frame $F^*$ in $F_r$ and the control problem we consider is to drive $F_0$ from its initial value
$F_0(0)$ to $F^*$ in a stable way from image information.  If the controlled variables are a parametrization of
the euclidean transformation (pose) between $F_0(t)$ and $F^*$, the control is called {\em 3D visual servoing}.
When using a rigid set of $n$ 3D lines $(\L_i)$, this requires to use a dedicated  pose computation method,
like in \citep{dhome89a,chen91b}. On the contrary, when using only an image-based parametrization
of the projected lines, as done in \citep{chaumette90}, we fall in the category of 2D visual servoing.
The first approach allows to accurately control the trajectories in space, making for example easy the
avoidance of obstacles. However, this may lead to loosing the features in the image. Moreover, the
analysis of the control performances is uneasy because of the insertion of the pose
computation algorithm in the loop. The pure 2D control scheme has of course opposite
characteristics and may further exhibit singularities.  In order to overcome these drawbacks, the method we propose is a combination 
of the two approaches, which could be called 2D 1/2 visual line servoing. Indeed, for every line, the representation we use
allows to extract directly 2  parameters from the image, i.e. $\hn$, the normal to the interpretation plane. The
 remaining 2 parameters may be considered as 3D ones: the depth $h$ and the last component of $\un$. We will see
later how to take them into account. Finally, our goal will therefore be defined as the set
$\{ \un_i^*, \hn_i^*, h_i^* ~~;~~ i=1\ldots n \}$.

We will firstly demonstrate the principle of the control scheme we propose by using a single line. We will then consider the more general case of the $n$ lines
as just defined, and finally focus on the particular situation of the orthogonal trihedron. Before,
let us set the assumptions which will be used in all the following.
\begin{itemize}
\item The set of goal lines $\L_i^*$ corresponds to a  fixed achievable configuration of the system (i.e. the
displacement between the current position and the goal one is a rigid motion).
\item All 3D trajectories of the camera which are needed are feasible (no obstacles, no joint limits
neither geometrical singularities are reached by the device which moves the camera).
\item During all the motion, no line $(\L_i)$ goes through the optical center of the camera.
\item The camera is fully calibrated.
\item All the needed variables but the depth are exactly known or measured.
\item The velocity screw $(\Omega, V)$ of the camera is the control variable (i.e. dynamical
effects are neglected).
\end{itemize}

\subsection{The Example of a Single Line}

 Let us consider a single target line specified by the set $\{ \un^*, \hn^*, h^*\}$ and forget for the moment the 
depth $h$. The state equations of this subsystem are given by
eqs~(\ref{du})~and~(\ref{dhn}) only. We may remark that the   
orientation part is independent from  the translation part, which
allows to consider that this system is in
cascaded form. Since the manifold of lines does not exhibit natural metric we have to define convenient
errors. By exploiting the above mentioned partial decoupling, we set
\begin{equation}
\label{eu}
\bf{e}_u= \un \times \un^*
\end{equation}
and
\begin{equation}
\label{eh}
\bf{e}_h= \un \times (\hn -\hn^*)
\end{equation}
We can therefore set the following cascaded control:
\begin{equation}
\label{cde1droite_np}
\left\{ \begin{array}{l}
\Omega= \mu_u \bf{e}_u \\
V= -\mu_h \bf{e}_h |_{\Omega =0; \un=\un^*}
\end{array}
\right.
\end{equation}
 where $\mu_u > 0$, $\mu_h >0$.
Remarks:
\begin{enumerate}
\item $\bf{e}_u$ is the classical error on the sphere; it allows to move on a geodesics.
\item $\bf{e}_h$ is a better choice than simply $(\hn -\hn^*)$ since it can be shown that it ensures
that $\hn \times \hn^*$ is minimum even if $\Omega$ vanishes without the right
equilibrium reached ($\un_{eq} \neq \un^*$) (see
\citep{TheseNiko}). Moreover, since $\hn$ is a unit vector, it is
constrained to move on a sphere. Thus, $\bf{e}_h$ allows to move
on this sphere using translational degrees of freedom only
(figure~\ref{illusV}).
\item The evolution equation of (\ref{dhn}) under the control (\ref{cde1droite_np}) is:
\begin{equation}
\dot{\hn}=  \frac{\mu_h }{h} {\bf{e}_h}^T \hn (\un \times \hn)
\end{equation}
This shows that the unknown depth\footnote{Recall that $h$ can be
considered as a depth \S~\ref{bpc}.}, $h$, acts only as a
variable unknown positive gain, which has no 
consequence on the stability in continuous time.

\item Since we don't consider $h$, the null space of the system is 3-dimensional. Instead of (\ref{ns1}),
(\ref{ns2}), it is given by
\begin{align}
\label{ns3}
\Omega_{ns} &= \lambda \un\\
\label{ns4}
V_{ns} &= \mu \un + \lambda h \hn + \nu (\un \times \hn)
\end{align}
where $\mu$, $\lambda$ and $\nu$ are real numbers. Any motion within this subspace will not affect 
the stabilization.
\end{enumerate}

\begin{figure*}[t]
\psfrag{h}{$\hn$}
\psfrag{h*}{$\hn^*$}
\psfrag{dh}{$\dot{\hn}$}
\psfrag{V}{$\Vlin=-\mu_h \vect{e}_h$}
\psfrag{u}{$\unit{u}$}
\psfrag{P}{$P$}
\psfrag{e}{$\vect{e}_h=\unit{u} \times (\hn - \hn^*)$}
\centerline{\resizebox{!}{0.4\textwidth}{\includegraphics{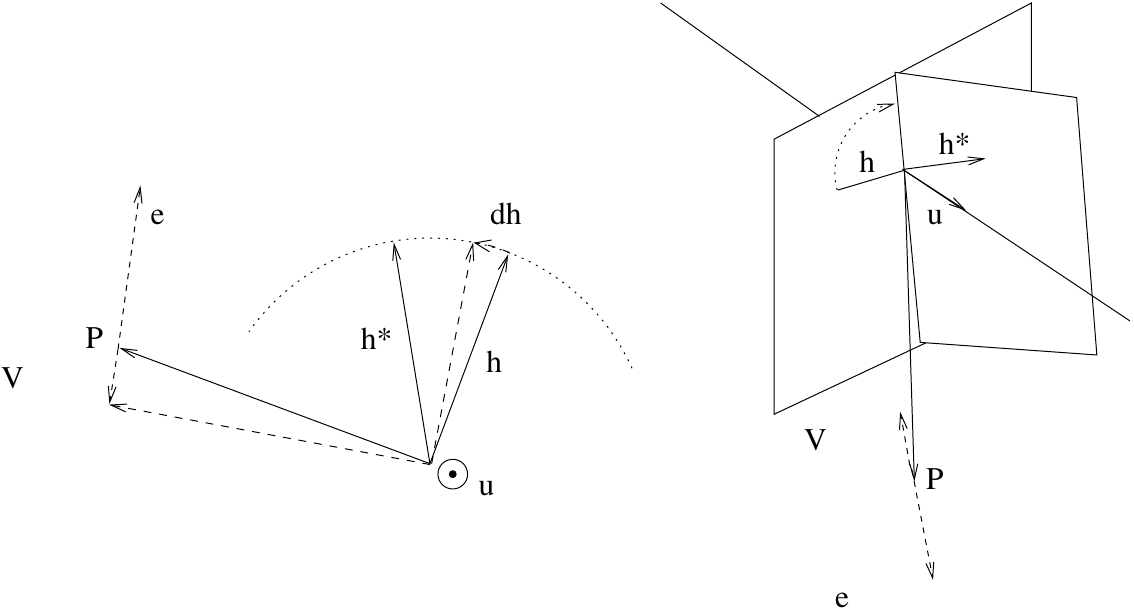}}}
\caption{\label{illusV} The action of the translation control 
$\Vlin=- \mu_h \unit{u} \times (\hn - \hn^*)$~: in the case where
$\unit{u} = \unit{u}^*$ (left) and in the general case (right).}
\end{figure*}

Considering the latter remark, we notice that the translation
control in~(\ref{cde1droite_np}) may have components in the two dimensions of the
null space ($\mu$,$\nu$) that are independent from the rotation
control. These motions are useless  disturbances unless they are used
for complementary tasks (see section \ref{laser}).
It is therefore  pertinent to project the linear velocity on the only relevant component, i.e. $\hn$
when $\Omega =0$. 
Hence, we obtain the final cascaded control for one line:
\begin{equation}
\label{cde1droite}
\left\{ \begin{array}{l}
\Omega= \mu_u \bf{e}_u \\
V= -\mu_h \epsilon_h \hn |_{\Omega =0; \un=\un^*}
\end{array}
\right.
\end{equation}
 where $\epsilon_h = (\un \times \hn)^T\hn^*$, $\mu_u > 0$, $\mu_h >0$.
We have the following result:
{
\proposition 
The control (\ref{cde1droite}) is such that:\\
1- if $\un(t=0) \neq -\un^*$, then  $\un(t)$ asymptotically converges to $\un^*$;\\
2- if  $\hn(t =0|_{\Omega =0; \un=\un^*}) \neq -\hn^*$, then $\hn(t)$ asymptotically converges to~$\hn^*$.
}

The proof is easy (see \citep{TheseNiko}) using
$\frac{1}{2} || \un-\un^*||^2$  and  
$\frac{1}{2} || \hn-\hn^*||^2$ as Lyapunov   
functions respectively.

In practice, the stabilization in orientation can be considered as
``sufficiently'' achieved after a finite time $T_{\Omega}$. The
translation control then writes~as 
\begin{equation}
V= -\mu_h s(t-T_{\Omega}) \epsilon_h \hn
\end{equation}
where $s(t)$ is the Heaviside step function, and the convergence condition becomes
 $\hn(t=T_{\Omega}) \neq -\hn^*$.
Further remarks:
\begin{enumerate}
\item In practice, the two control vectors may not be fully cascaded: the control in orientation should be
activated all the time, because of unmodelled coupling effects between translational and rotational motions;
the control in translation can be progressively activated using a transition function smoother than the step one,
in order to avoid unexpected large motions. This is the reason why we wrote $\un$
instead of $\un^*$ in the proposition.

\item The forbidden initial conditions in the proposition are unstable equilibriums. They are easy to avoid.
\end{enumerate}

\subsection{General Case}
We now consider the case of a set of $n >1$ distinct lines, rigidly linked, and for the moment we
do not assume further properties on their configuration. Let us recall that the equations of their motion are, for $i=1 \ldots n$:

\begin{align}
\label{dui}
\dot{\unit{u}_i} &= \Omega \times \unit{u}_i\\
\label{dhni}
\dot{\hn_i} &= \Omega \times \hn_i - \frac{V^T \hn_i}{h_i}(\unit{u}_i \times \hn_i)
\end{align}
Again, we left the $h_i$'s, the dynamics of which are given by eq (\ref{dhn}), free to evolve. 
 Like previously, we define an ideal
cascaded control, allowing us therefore to split it in two parts: the rotational, then  the
translational, provided that the former is stabilized. 
\subsubsection{ Control in Rotation}
It is a straightforward generalization of the single line case which results in 
\newtheorem{theorem}{Theorem}
{\theorem

\label{theoreme_rotation}
 (proof in appendix~\ref{preuve_theo_rot})\\
If $\un_i (t=0) \neq -\un_i^*, ~\forall i=1 \ldots n$,

Then,  the control
\begin{equation}
\label{rotation_control}
\Omega = \mu_u \sum_{i=1}^n \un_i \times \un_i^*, ~~~\mu_u >0
\end{equation}
 asymptotically stabilizes the equilibrium of the system (\ref{dui}), $i=1 \ldots n$:
\begin{equation}
\un_i=\un_i^*, ~~~ i=1 \ldots n
\end{equation}
}
\subsubsection{Control in Translation}
The translational control presented for a single line can be
generalized in a similar way. The following result holds:

{\theorem
\label{theoreme_translation}
 (proof in appendix~\ref{preuve_theo_trans})\\
If for all $i=1..n$, the vectors $\unit{u}_i$ are constant (but not
necessarily equal to the vectors $\unit{u}_i^*$),

Then:
\begin{enumerate}
\item the control 
\begin{equation}
\label{cde_transl}
V=-\mu_h \sum_{i=1}^n \matx{B}_i\epsilon_i \hn_i~ |_{\Omega =0; \un_i=\un_i^*\forall i},~~~\mu_h >0
\end{equation}
$\matx{B}_i$  being $(3 \times 3)$ weighting matrices.
stabilizes equilibrium points under the condition:
\begin{equation}
\left( \sum_{i=1}^n \matx{B}_i \epsilon_i \hn_i
\right)^T
\left(  \sum_{i=1}^n \frac{ \epsilon_i  \hn_i }{h_i}\right) > 0
\end{equation}
\item
 the  choice $\matx{B}_i=\frac{1}{h_i}\id{3}, \forall
i=1..n$, which  corresponds to the control:
\begin{equation}
\label{corrolary_control}
V=-\mu_h \sum_{i=1}^n \frac{1}{h_i}\epsilon_i \hn_i~ |_{\Omega =0; \un_i=\un_i^*\forall i},~~~\mu_h >0,
\end{equation}
stabilizes the stationnary points:

\begin{equation}
\sum_{i=1}^n \frac{ \epsilon_i \hn_i }{h_i} = 0
\end{equation}
whatever the line configuration.
\end{enumerate}
}

However, the analytical expression of the stationnary points cannot have an easy
 geometrical interpretation in any  line
configuration. Additionnal knowledge on the system geometry is therefore required.
We will see in the following that the choice of a particular object, the
orthogonal trihedron, will allow us to obtain  results stronger than in the general case.

\subsection{Case of the Orthogonal Trihedron}
\label{triedre}

The trihedral configuration of lines is of particular interest since it is very often
present in structured environments like in buildings or in manufactured systems. However, from the image 
point of view, it has been shown (see \citep{dhome89a}) that any  pencil of lines (i.e. with
a common intersection) is degenerated for pose computation. Indeed, for example, it is intuitive
that observing an ideal infinite trihedron with a single camera doesn't allow to recover its depth:
imagine that the projected trihedron is centered in the image; then any
translation along the optical axis leaves the projection unchanged. 
In the next section, we will propose a practical solution to this problem. For the moment, let us
study
the control without taking the depth into account. The final equilibrium will therefore
have a single degree of freedom left.

With respect to the general case, the following constraints have to be considered for the
orthogonal
trihedron:
\begin{itemize}
\item $n=3$
\item $\un_i^T \un_j =0 ~\forall~i \neq j$ (orthogonality)
\item $\vect{h}_i^T\un_j +\vect{h}_j^T\un_i =0 ~\forall~i, j$ (intersection)
\end{itemize}

Taking these constraints into account allows to set the following results 
\begin{itemize}
\item Rotation: theorem 1 is still true.
\item Translation:
%, standard case:
%{\theorem (proof in appendix)\\
%\label{theoreme_triedre}
%If $\hn_i (t=0) |_{\Omega =0; \un_i=\un_
%i^*~i=1,2,3} \neq -\hn_i^*,~ ~i=1,2,3$
%
%Then the system (\ref{dhni}, $i=1,2,3$), with the control
%\begin{equation}
%V=-\mu_h \sum_{i=1}^3  \un_i \times (\hn_i-\hn_i^*)~ |_{\Omega =0; \un_i=\un_
%i^*~i=1,2,3}, ~~~\mu_h >0
%\end{equation}
%is aymptotically stabilized to its equilibrium
%\begin{equation}
%\hn_i=\hn_i^*,~ ~i=1,2,3
%\end{equation}
%}
%\item Translation, projected case:
%The control writes in that case as:
%\begin{equation}
%\label{Vprojt}
%V=-\mu_h \sum_{i=1}^3 ((\un_i \times \hn_i)^T \hn_i^*)\hn_i~ |_{\Omega =0; \un_i=\un_
%i^*~i=1,2,3}, ~~~\mu_h >0
%\end{equation}
%
It can be shown (cf \citep{TheseNiko}) that the stationnary points of
the control~(\ref{corrolary_control}) are the configurations where
 a plane of the trihedron intersects  the optical center of the
camera. As a consequence, the control is asymptotically stable
provided that the trajectory of the camera stays in a
single octant.
\end{itemize}

However, it can be shown that the control~(\ref{corrolary_control})
is inversely proportionnal to the depth of the trihedron center. Since this
depth can not be observed, we can set it to an arbitrary value (for instance,
the desired depth after convergence). But, in the particular case of the
orthogonal trihedron a simpler control law can be used:

{\theorem
\label{control_trihedron}
In the case of the orthogonal trihedron and
if $\hn_i(t=0)\neq -\hn_i^*$, then  the following control
\begin{equation}
\label{simple_control}
V=-\mu_h \sum_{i=1}^n \epsilon_i \hn_i~ |_{\Omega =0; \un_i=\un_i^*\forall i},~~~\mu_h >0
\end{equation}
is independent from the depth and asymptotically stabilizes the 
equilibrium of the system~(\ref{dhni})
}\\
The proof of this theorem is given in Appendix~\ref{control_trihedron_proof}.

%this control exhibits  stability properties
%  similar to the previous ones, provided that the trajectory of the camera stays in a
%single octant. In fact, a configuration where a plane of the trihedron intersects
% the optical center
%of the camera corresponds to a singularity of the system with the control (\ref{Vprojt}).

Let us mention that, like in the case  of a single line, the practical implementation of
these control schemes can be achieved using some overlapping of the rotational and translational
parts. Simulations  of section \ref{results} will provide some results concerning this matter.

\section{Practical Achievements}
\label{results}
\subsection{Background: a Welding Application}
This work was performed in  the framework of a European project, VIGOR \citep{vigor}. The main
industrial partner was the Odense Steel Shipyard in Denmark and a major concern 
of the project was to improve  the
automatic welding of ship parts. In particular, it was required to
accurately move a welding tool with respect to a ship part, the position of
which is not exactly known.

The state-of-the-practice combines off-line trajectory generation
based on CAD models and on-line following of the generated
trajectory. This requires that the CAD models are perfect and that
there is no discrepancy between the position of the part with respect
to the robot base at trajectory generation time and at trajectory
following time. This is a strong constraint in the case of ship welding
since the size and the weight of the parts to be welded prevent from
achieving an accurate positioning of these parts.
Consequently, it was necessary to use sensor-based techniques, such as 
visual servoing.

Since 75\% of the parts to be welded are composed of three
intersecting orthogonal planes (see Figure~\ref{photo_triedre}), we
put the emphasis of the application of  our studies to   this junction of three
orthogonal lines. 

With respect to the theoretical control schemes previously presented, the following problems had to
be solved in order to achieve the implementation:
\begin{enumerate}
\item how to extract and track the lines in the image in a fast and robust way?
\item how to compute the orientations $\un_i$?
\item how to cope with the missing depth?
\end{enumerate}

We briefly present in the following the solutions we choose for all these issues, before
giving some simulation and experimental results.

\subsection{Adding Depth Control}
\label{laser}
Dhome {\em et al.}\citep{dhome89a} showed that a line junction is a
degenerated case of pose estimation in the sense that it is impossible 
to estimate the distance from the camera center point to the 
{\em junction center} (i.e. the
intersection of the \ddd lines) from a single image. Thus, one can only
estimate a {\em partial pose} of the camera: the \ddd orientation of
the lines and the line of sight passing through the junction center.

From the control point
of view, this degeneracy means that the distance from the robot to the
ship part cannot be measured neither controlled. Consequently, we need to add some other sensor to the
camera in order to observe the  depth of the junction center.

We rejected the use of stereovision for several reasons. Firstly, one
can verify that the line junction configuration requires the stereo
pair to be calibrated, which we would like to avoid, in order to
estimate the line junction depth. Secondly, to
achieve a good accuracy in this depth estimation, one should have an
accurate calibration {\em and} a wide baseline, which increases the
bulk of the robot end-effector. Finally, even though the depth were
not to be estimated but instead the two images were to be used to
implicitly define a depth, we would fall back into the troubles of
stereo visual servoing\citep{lamiroy00}.

The solution we chose was to use a commercial laser pointer (cheap and
simple) that we rigidly linked 
to the camera (figure~\ref{camera+laser}). In this configuration the image of the laser spot
moves on a given straight line whatever the observed scene and the
motion of the set \{camera-laser\}. 
This can be easily understood  if we consider the laser beam as the optical axis of a virtual
camera. The line in the image is therefore an epipolar line of this particular stereo rig. 

\begin{figure*}[t]
\centerline{
\hfill
\resizebox{!}{0.4\textwidth}{\rotatebox{90}{\includegraphics{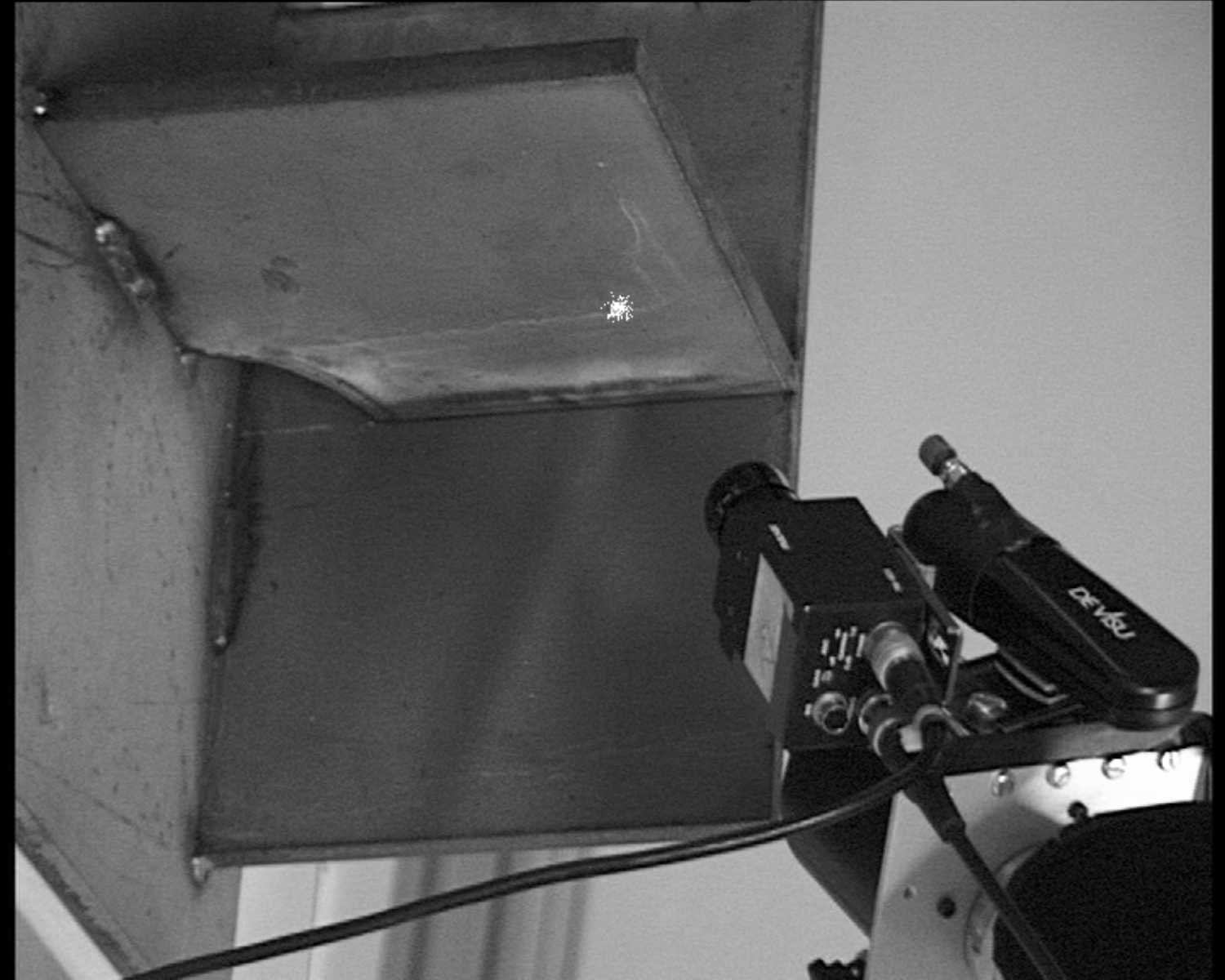}}}
\hfill
\resizebox{!}{5cm}{\includegraphics{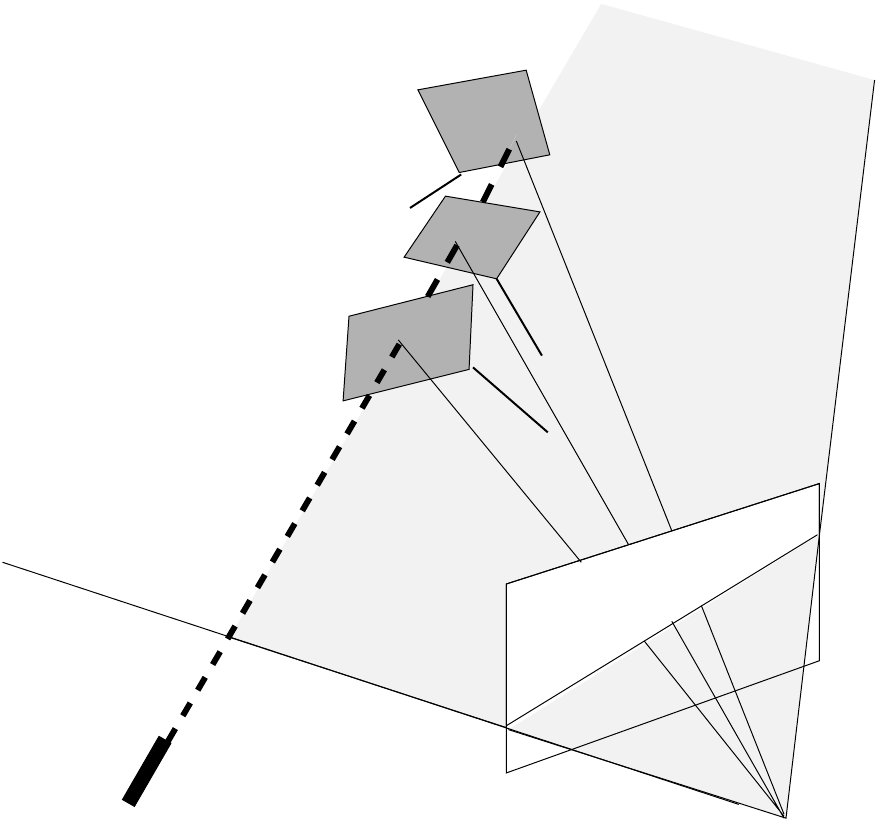}}
\hfill
}
\caption{A laser pointer is rigidly linked to the camera}
\label{camera+laser}
\end{figure*}

{\remark

If we calibrate the laser-camera system, we can recover the depth by triangulation. In order
to avoid  this step and for improving robustness, we will {\rm not} consider that the laser-camera
system
is calibrated. We will only need to know the laser epipolar line, which is straightforward to obtain
by recording the image spot position while moving, whatever the target.}

Let us now address the problem of depth control. In fact, we have up to now one dimension free.
We therefore need  to add to the control scheme one variable only, which should be
an image of the depth. Since the image laser spot moves on a line, it looks natural to choose
the signed distance of the spot from a given origin of the line, denoted as $s$, as the
variable to control. When the target is a plane, $s$ is a monotonic function of the true depth. In order to design the control, we now need to know the variation of $s$,
i.e its interaction screw.
We have the following results:
{\proposition
\label{sdot}
If the laser spot moves in a single 3D plane, then the variation of $s$ is:
\begin{align}
&\dot{s}=  -\frac{h_L \cos \theta_z}{z_L^2} \\
& \times
\begin{pmatrix}
\tan\alpha_x & \tan\alpha_y & 1 -\frac{z_L}{\cos\theta_z} & \tan\alpha_y & \frac{z_L}{\cos\theta_z} & \tan \alpha_x 0 
\end{pmatrix} \nonumber \\
& \times \Theta_L \tau
\nonumber
\end{align}
where $\tau$ is the velocity screw $ ( V,\Omega )$, $\Theta_L$ is the adjoint operator
involving the rigid transformation $R_L, t_L$ between the camera frame
and the laser frame, $z_L$ is the depth of the laser spot along its view
line, $\alpha_x$ and $\alpha_y$ are two angles defining the normal to
the 3D plane in the laser frame, $h_L$ is the distance between the
camera center of projection to the intersection between the laser beam
and the image plane, and $\theta_z$ is the angle between the laser
beam and the optical axis.
}

Notice that this equation can be derived from eq.~(7.1.48) in
\cite[pp~273--274]{samson91}, obtained while defining the
interaction screw of general thin-field
range-sensors. An alternative proof of those results is given
in~\citep{TheseNiko}, with explicit references to the present case. 

{\proposition 

In addition, if the rotational velocity is zero, and if the translational velocity is
such that $V_l=v \un_l$ where $\un_l$ is a constant unit vector, then
\begin{equation}
\label{sdot2}
\dot{s}=\frac{v}{z^2} \left (-h_L k_z (\tan \alpha_x ~~ \tan \alpha_y ~~ 1~~) R_l \un_l \right)
\end{equation}
where the term in brackets is a constant, noted $k$.}

From this result, it is now straightforward to design the depth control.  In order to take
benefit of the last proposition, the depth control will be also cascaded with the rotational
one. Let us therefore assume that the desired orientation has been reached. Then, $\Omega =0$,
even though it is highly advisable to keep this control activated in order to ensure that the
orientation remains constant all the time despite disturbances. Let us choose as $\un_l$ the direction
of the view line going through the trihedron center. An image of the depth error is
$e_s = s- s^*$ where $s^*$ is the desired image spot position on the laser epipolar line, and with, from (\ref{sdot2}), $\dot{e}_s = kv/z^2$. An ideal exponential
behavior would be reached by choosing $v$ as 
\begin{equation}
v = -\lambda \frac{z^2}{k} (s-s^*),~~~\lambda >0
\end{equation}

Since $z$ is unknown, but is always nonzero, it is possible to fix its value to an
arbitrary one. Stability will not be affected, but the convergence will no more be exponential.
The depth control part is finally:
\begin{equation}
V_l= -\frac{\mu_l}{k} (s-s^*) \un_l,~~~\mu_l >0
\end{equation}

It remains only now to integrate this control in the translational
one.  Recall that we had the following control:
\begin{equation}
V_h=-\mu_h \sum_{i=1}^3 ((\un_i \times \hn_i)^T \hn_i^*)\hn_i, ~~~\mu_h >0
\end{equation}
The final translational velocity control is simply $V= V_h +V_l$. It can indeed
be  seen that the control $V_h$ is a combination of the
$\hn_i$'s, which are all orthogonal to $\un_l$ owing to the trihedron
geometry (intersection). Therefore $V_l$
is in the null space of $V_h$ and can run simultaneously with it.

\subsection{Line Extraction and Tracking}

To extract the lines and track them in the image along the
motion, we used an adapted version of the \dd-\ddd model-based
approach proposed in~\citep{marchand98a} and which is summarized by
their authors as: 
``{\em in a first step, the object image motion is represented by a 2D
affine motion model, and is estimated, using a robust statistical
method, from the computation of the normal displacements evaluated
along the projected model contours. [...] The 2D affine motion model
does not always match the real displacement of the object. A second
step that consists in fitting the projection of the object model on
the intensity gradients in the image is necessary. This is achieved
using an iterative minimization of a non-linear energy function with
respect to 3D pose parameters.}''

\subsection{Computation of the Line Orientations}

To compute the line orientations, we based ourselves on the analytic
solution to the perspective 4 point problem proposed
in~\citep{horaud89b}. 
Indeed, the orthogonal trihedron is a particular case of the 4 point
problem where 3 of the points lie at the infinite.

Thus, using $\un_l$ the unit vector
of the view line going through the trihedron center, we define
$\unit{v}_i=\un_l \times \hn_i$. Recall that the $\hn_i$'s can be
expressed as:
\begin{equation}
\hn_i = \frac{\un_l \times \un_i}{\|\un_l \times \un_i\|}, \quad i=1..3
\end{equation}
Consequently, if we give the same orientation to the \ddd lines and to their projections,
we obtain the following system:
\[
\unit{u}_i = \cos \theta_i \un_l + \sin \theta_i \unit{v}_i,\quad 
i=1..3, \quad \sin \theta_i > 0
\]
where the unknowns are the angles $\theta_i, i=1..3$. This system is
similar to the one obtained in~\citep{horaud89b} and its solution is
of the form:
\begin{eqnarray*}
\cos \theta_i &=& \epsilon_i \sqrt{\frac{\unit{v}_{i-1}^T\unit{v}_{i+1}}
{\unit{v}_{i-1}^T\unit{v}_{i}\unit{v}_{i}^T\unit{v}_{i+1}+\unit{v}_{i-1}^T\unit{v}_{i+1}}}\\
\sin \theta_i &=& \sqrt{\frac{\unit{v}_{i-1}^T\unit{v}_{i}\unit{v}_{i}^T\unit{v}_{i+1}}
{\unit{v}_{i-1}^T\unit{v}_{i}\unit{v}_{i}^T\unit{v}_{i+1}+\unit{v}_{i-1}^T\unit{v}_{i+1}}}
\end{eqnarray*}
where $\epsilon_i=\pm 1$ depending on whether the visible half of the
\ddd line points inwards or outwards. In the case we deal with, where
the trihedron faces are opaque, this sign is identical for all
the 3 lines.
Notice finally that to simplify the notation, the operations $i+1$ and
$i-1$ represent the addition and substraction modulo 3.

\subsection{Simulation Results}

We present here some simulation results to show the behaviour of the
control law we proposed with both the use of lines and the depth
control. We compare the sequential activation of the control in
rotation and then in translation
(Figures~\ref{simu:cascade:rot+trans},~\ref{simu:cascade:erreurs}~and~\ref{simu:cascade:traj})
to their simultaneous activation (Figure~\ref{simu:simult:rot+trans},~\ref{simu:simult:erreurs}~and~\ref{simu:simult:traj}).
In both cases, depth control is always activated after the desired
partial pose is obtained.

% Images cascade
\begin{figure*}[tb] 
\centerline{
  \resizebox{!}{5cm}{
   \includegraphics{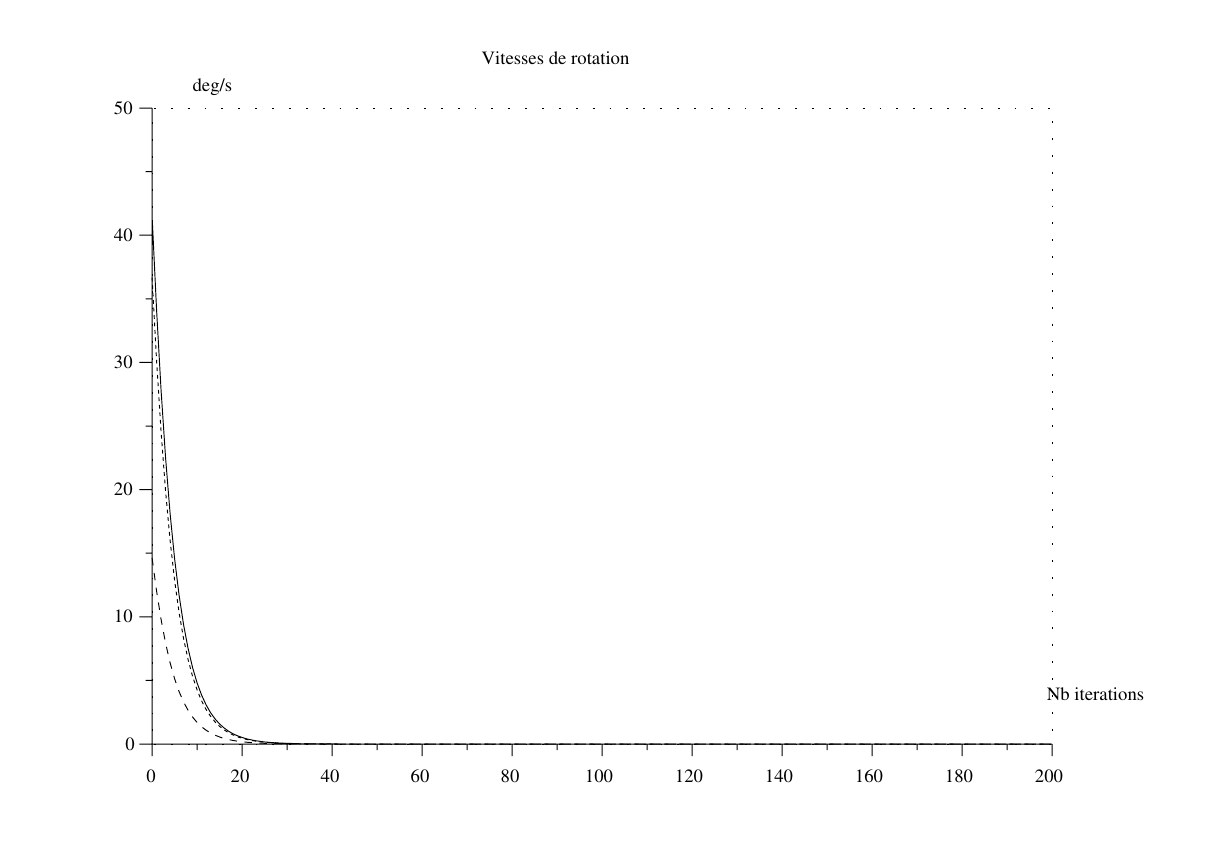}
  }
  \resizebox{!}{5cm}{
   \includegraphics{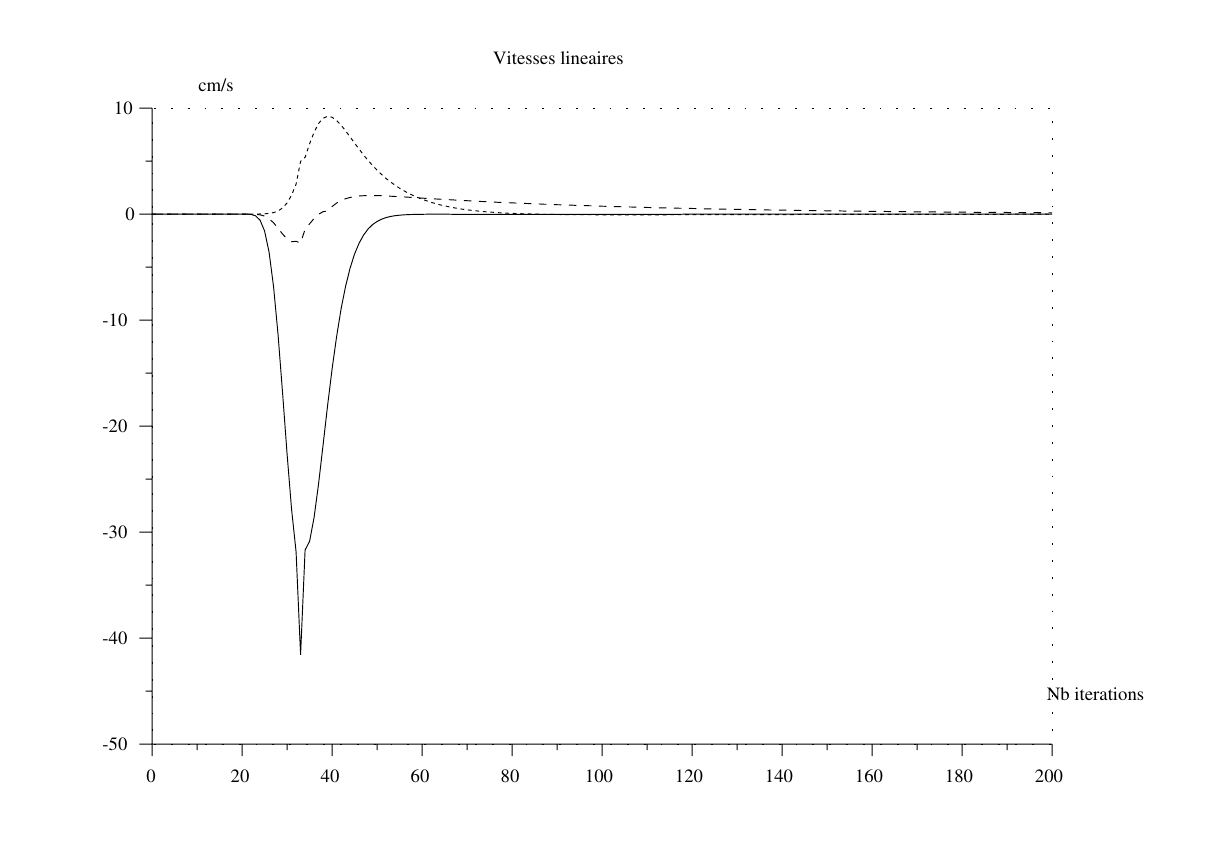}
  }
}
\caption{\label{simu:cascade:rot+trans} Sequential activation of
the control in rotation and then in translation: Rotation (left) and
translation (right) control.}
\end{figure*}

\begin{figure}[tb]
\centerline{
  \resizebox{!}{5cm}{
   \includegraphics{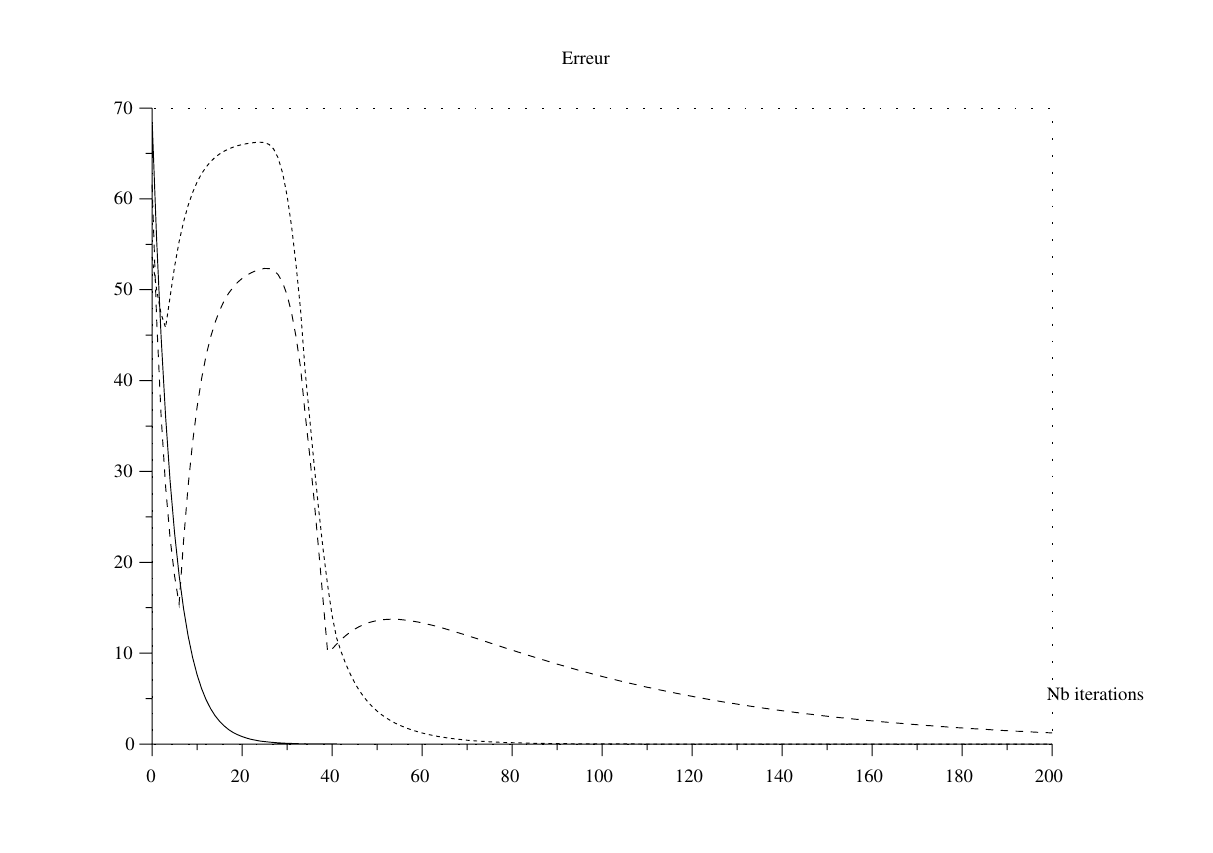}
  }
}
\caption{\label{simu:cascade:erreurs} Sequential activation of
the control in rotation and then in translation: Errors in orientation
(solid line), on the image line projections (dotted line) and on the
laser spot position (dashed  line).}
\end{figure}

\begin{figure*}[tb]
\centerline{
  \resizebox{!}{5cm}{
   \includegraphics{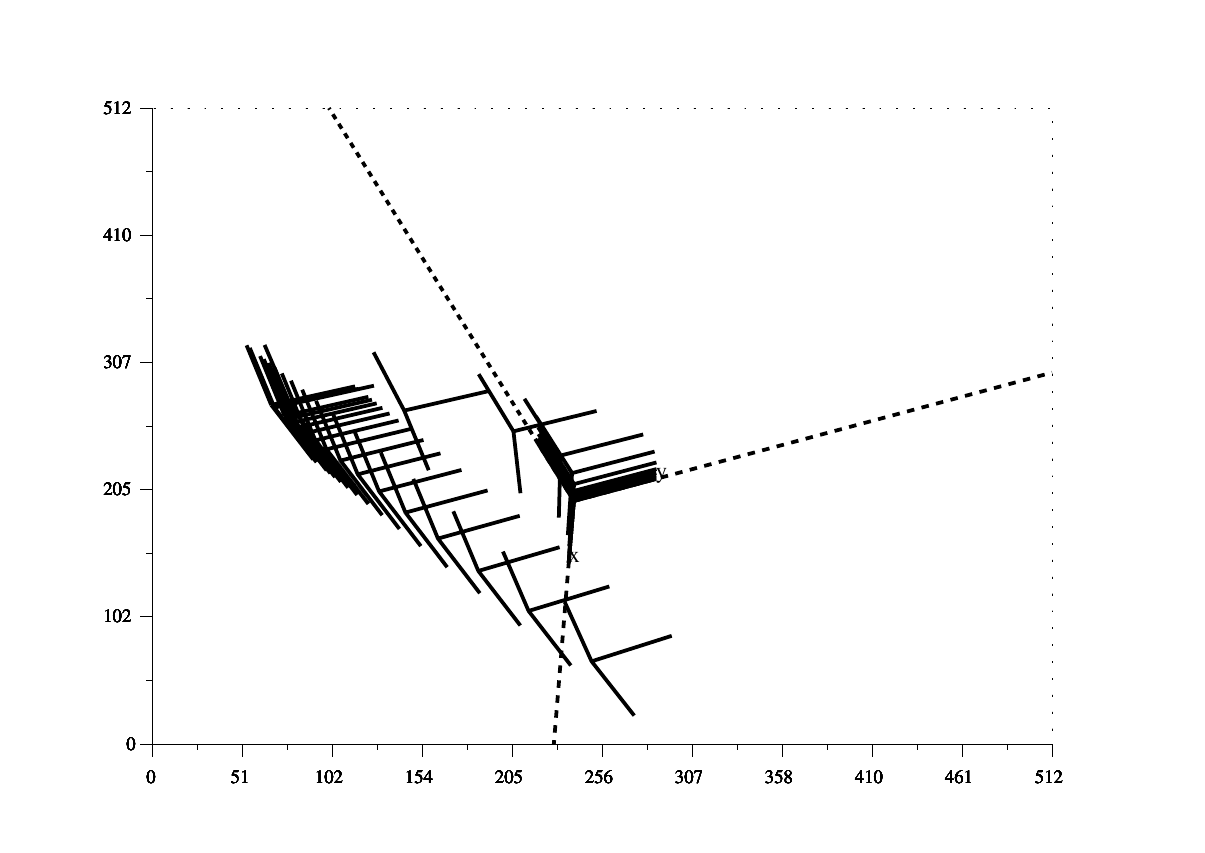}
  }
  \resizebox{!}{5cm}{
   \includegraphics{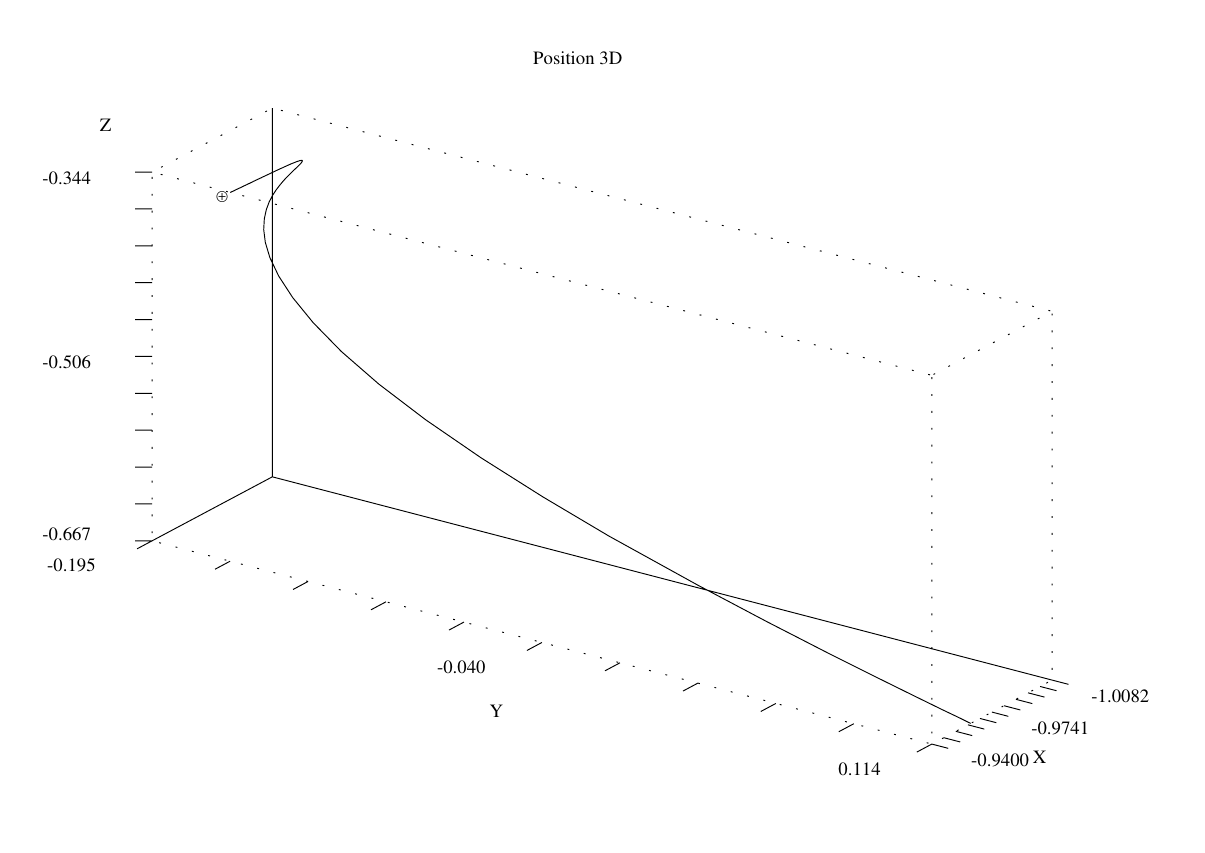}
  }
}
\caption{\label{simu:cascade:traj} Sequential activation of
the control in rotation and then in translation: Trajectory of the
trihedron in the image (left) and in the \ddd space (right).}
\end{figure*}

% Images Simult
\begin{figure*}[tb] 
\centerline{
  \resizebox{!}{5cm}{
   \includegraphics{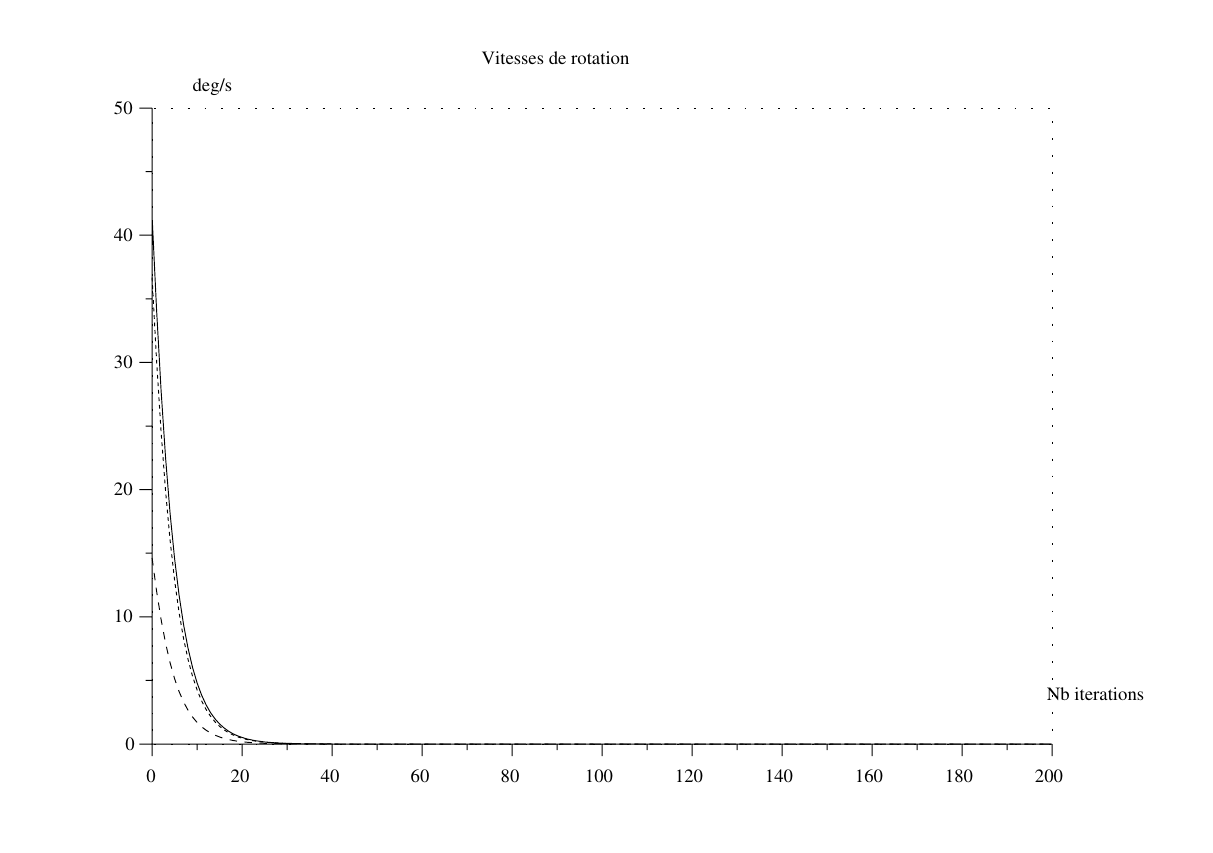}
  }
  \resizebox{!}{5cm}{
   \includegraphics{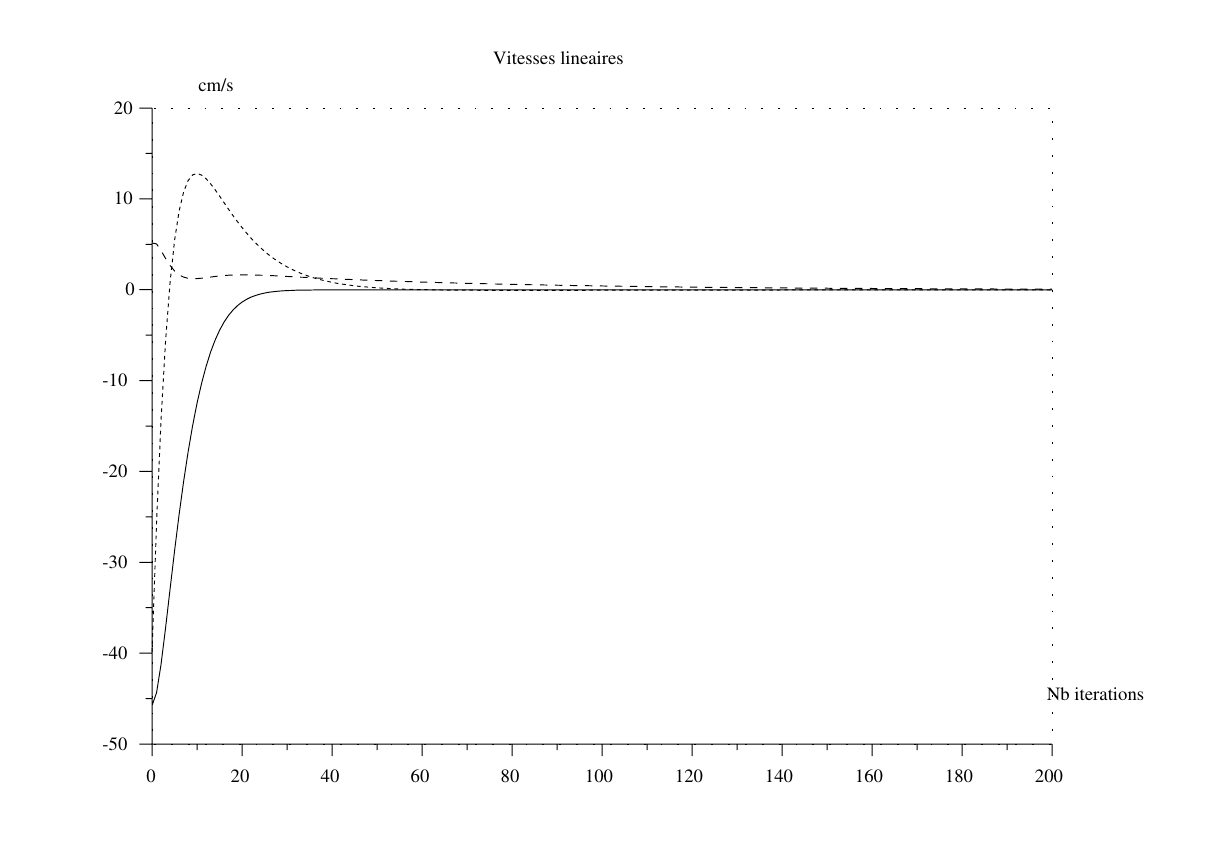}
  }
}
\caption{\label{simu:simult:rot+trans} Simultaneous activation of
the control in rotation and then in translation: Rotation (left) and
translation (right) control.}
\end{figure*}

\begin{figure}[tbp]
\centerline{
  \resizebox{!}{5cm}{
   \includegraphics{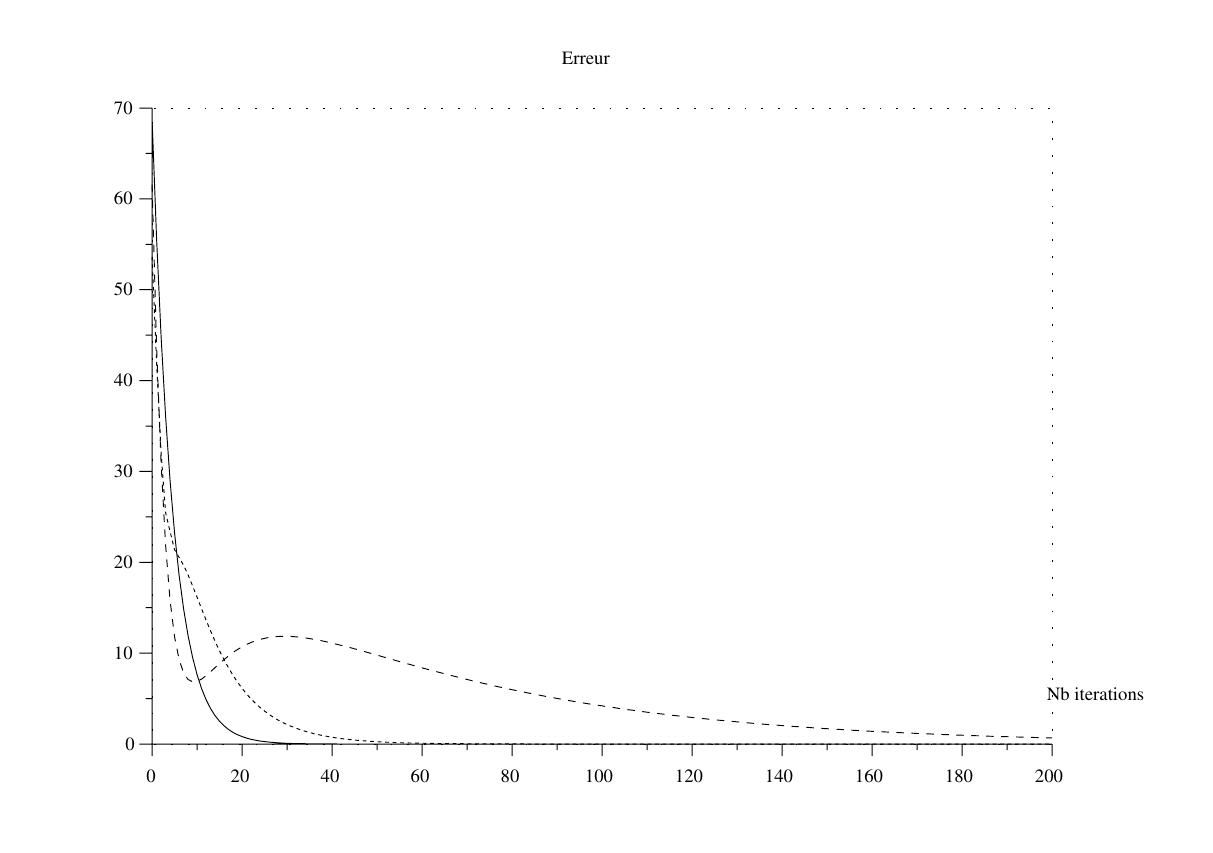}
  }
}
\caption{\label{simu:simult:erreurs} Simultaneous activation of
the control in rotation and then in translation: Errors in orientation
(solid line), on the image line projections (dotted line) and on the
laser spot position (dashed  line).}
\end{figure}

\begin{figure*}[tbp]
\centerline{
  \resizebox{!}{5cm}{
   \includegraphics{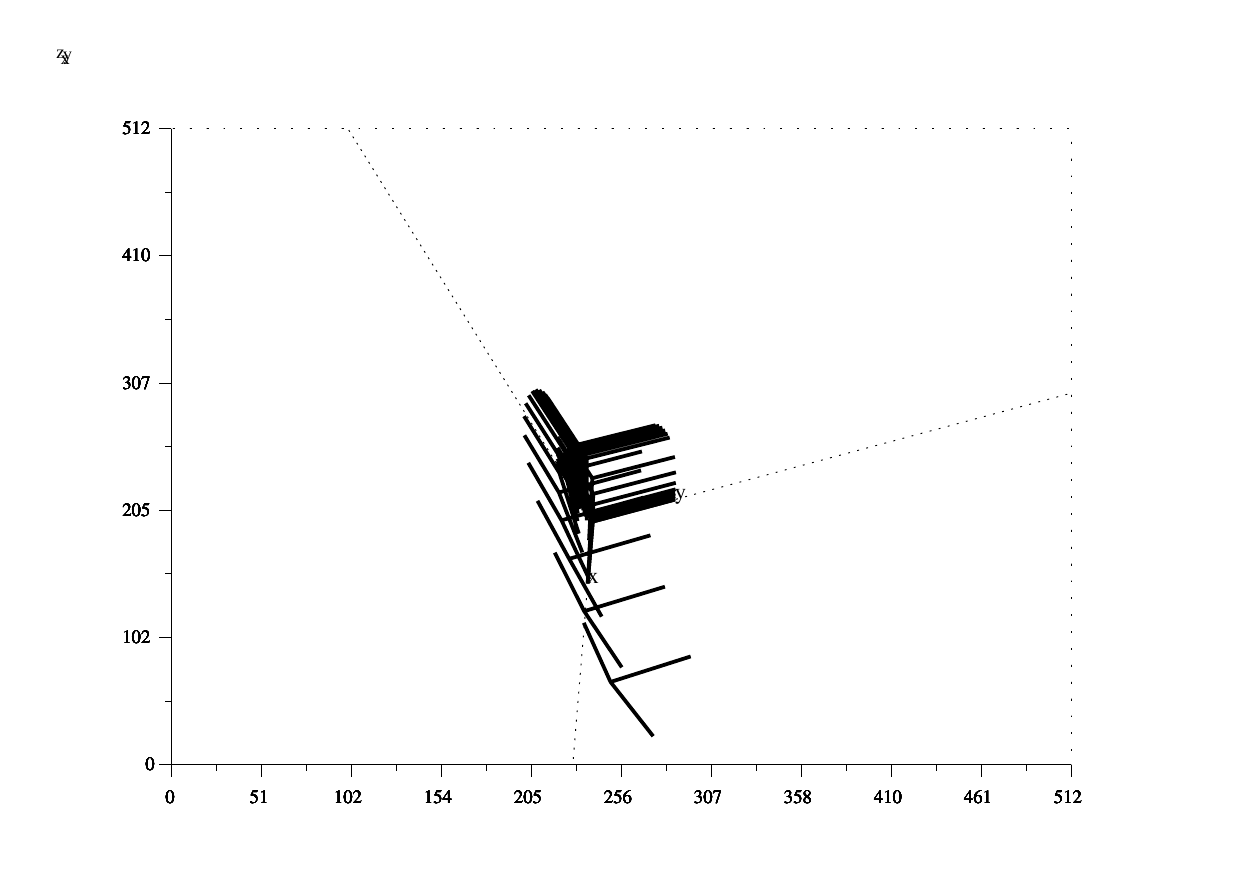}
  }
  \resizebox{!}{5cm}{
   \includegraphics{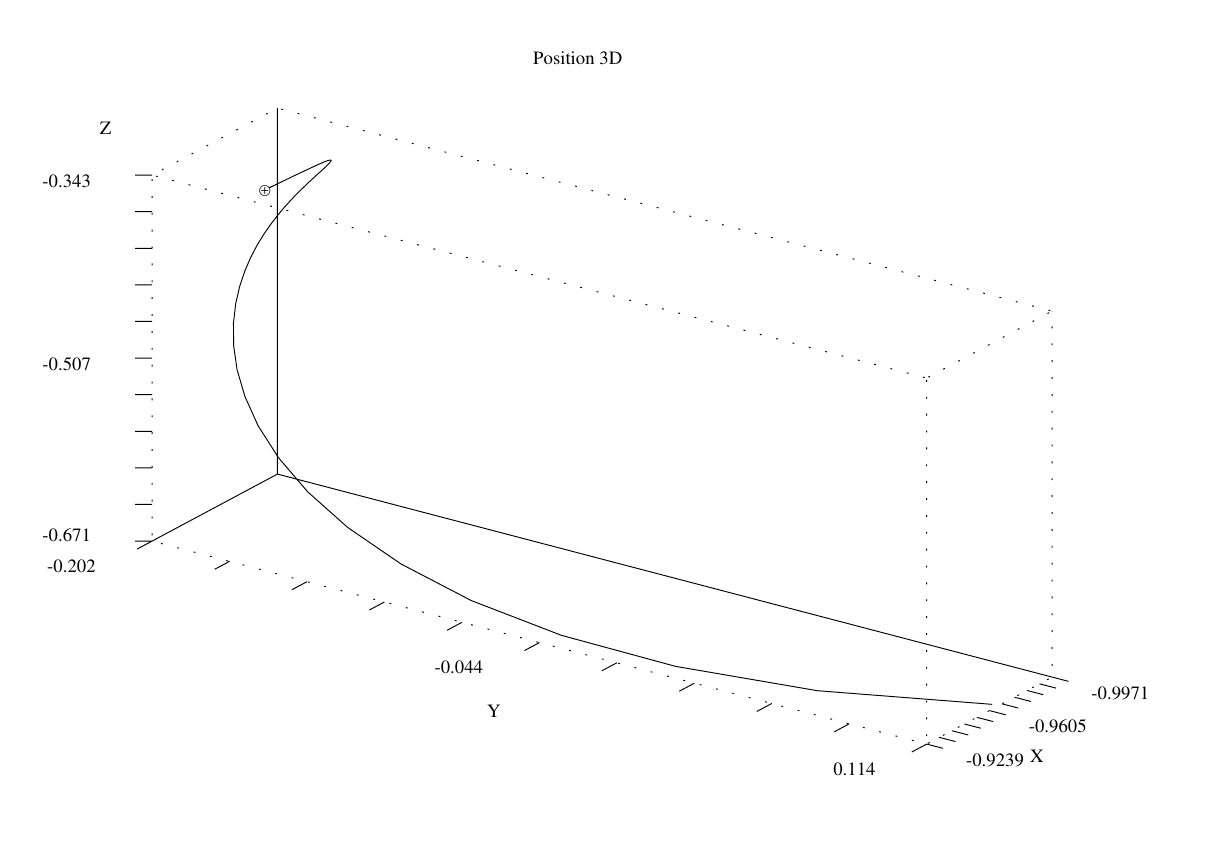}
  }
}
\caption{\label{simu:simult:traj} Simultaneous activation of
the control in rotation and then in translation: Trajectory of the
trihedron in the image (left) and in the \ddd space (right).}
\end{figure*}

Let us first consider the case of the sequential activation. In a
first time, only rotation control is activated
(Figure~\ref{simu:cascade:rot+trans}). It results in an 
exponential decay of the 
error in orientation together with an increase of the errors on the
image line projections and on the laser spot positions
(Figure~\ref{simu:cascade:erreurs}). Then, 
control in translation is activated, yielding the convergence of
the image line projection errors towards 0, without any influence on
the orientation. This can be seen in the image
(Figure~\ref{simu:cascade:traj}): the trihedron drifts to 
the left while orientation control is activated and then is brought
back to its desired position (large trihedron) by the control in
translation. 
Finally, depth control is
activated, bringing the laser spot to its desired position in the
image (Figure~\ref{simu:cascade:erreurs}), without the partial pose to be
modified. One can verify (Figure~\ref{simu:cascade:traj}) that this depth control
yields a straight line in the \ddd space.

Simultaneous activation of the control in rotation and translation can
be easily seen in Figure~\ref{simu:simult:rot+trans}. The partial
decoupling of the control 
law makes that the convergence in orientation is unchanged. On the
opposite, the simultaneous activation suppresses the drifting of the trihedron
(Figure~\ref{simu:simult:erreurs}). This results in a smaller amplitude of
the image motion  but in a larger one in the \ddd space
(Figure~\ref{simu:simult:traj}). Finally, one can see that the desired
partial pose is obtained sooner than in the sequential case and that
the depth control is thus activated sooner
(Figure~\ref{simu:simult:erreurs}).

\subsection{Experimental Results}

In this section, we focus on the experimental validation of the 
control law bases on lines. Thus, we disabled the depth control using
the laser beam and we simplified the visual tracking by using a
trihedron with sharp contrast rather than the ship part.

In all the experiments we led, the task was to bring back the camera
to a prerecorded position.

% Images translation pure
\begin{figure*}[htb]
\centerline{
\resizebox{!}{5cm}{\includegraphics{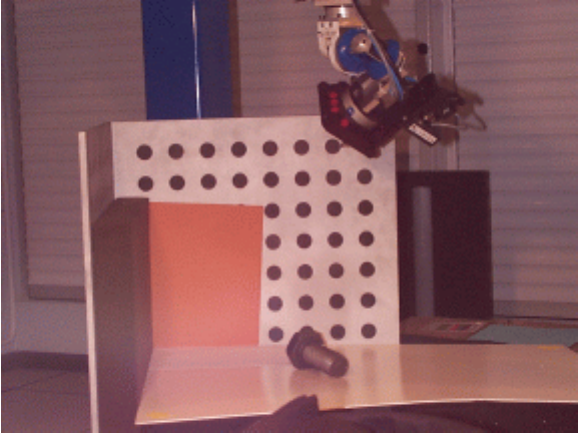}}
\resizebox{!}{5cm}{\includegraphics{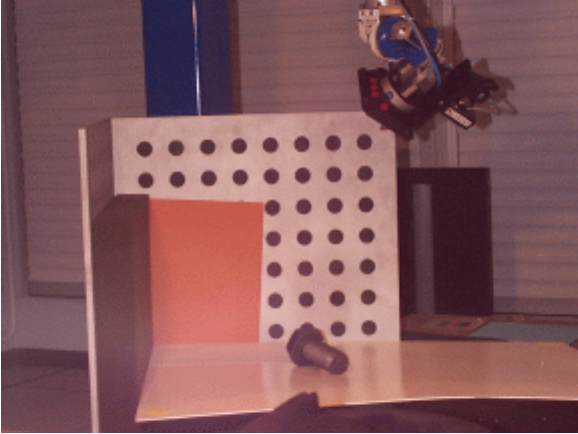}}
}
\caption{\label{consigne_trans} Pure translation:  Desired (left) and initial (right)
positions of the robot.}  
\end{figure*}

\begin{figure}[htb]
\psfrag{hi}{}
\psfrag{ui}{}
\centerline{%
\rotatebox{-90}{\includegraphics[width=0.21\textwidth]{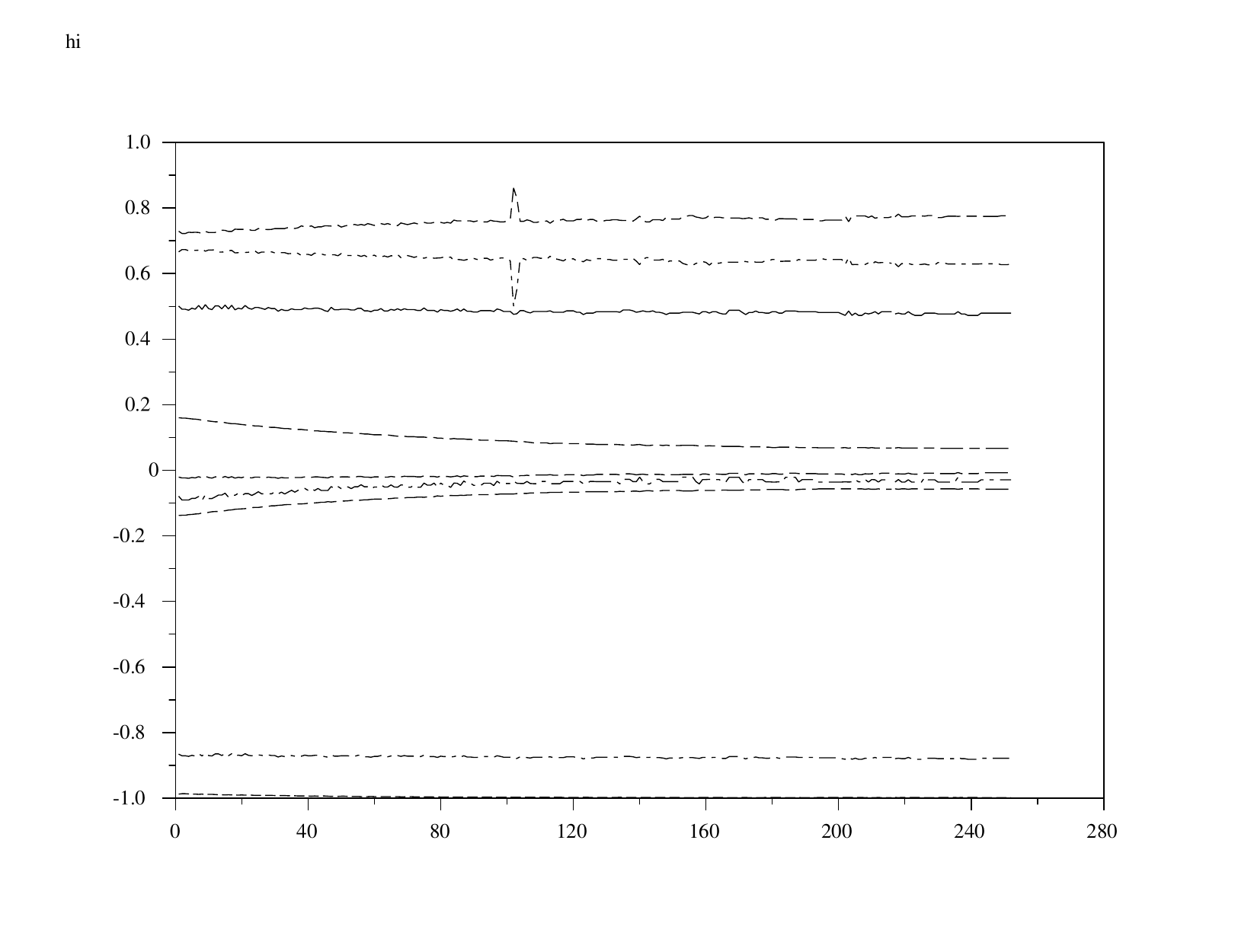}}
\rotatebox{-90}{\includegraphics[width=0.21\textwidth]{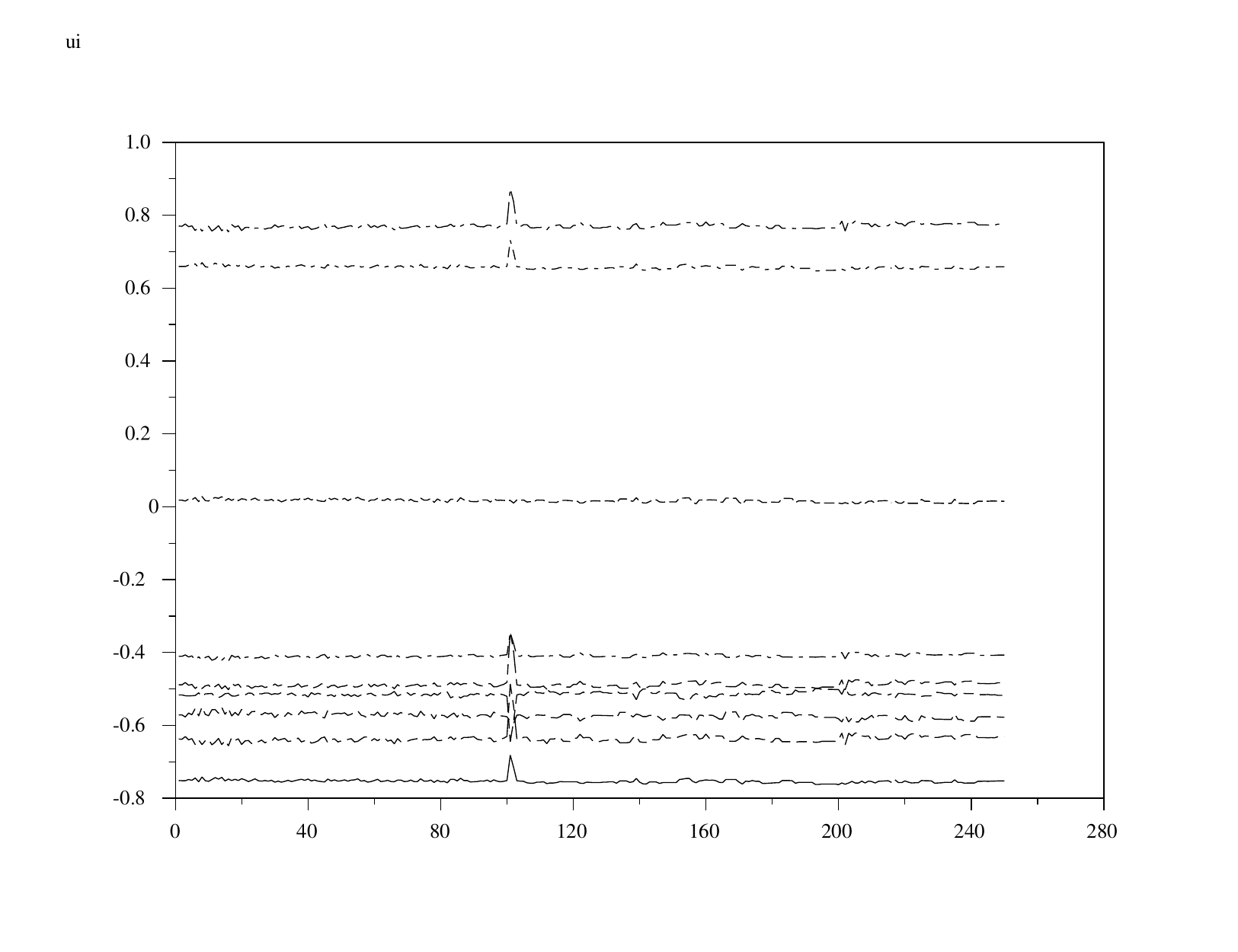}}
}
\caption{\label{transl_uh} Pure translation: coefficients of the vectors
$\hn_i$ (left) and $\unit{u}_i$ (right) along the control.}
\end{figure}
\begin{figure}[htb]
\psfrag{V}{ }
\psfrag{Omega}{ }
\centerline{%
\rotatebox{-90}{\includegraphics[width=0.21\textwidth]{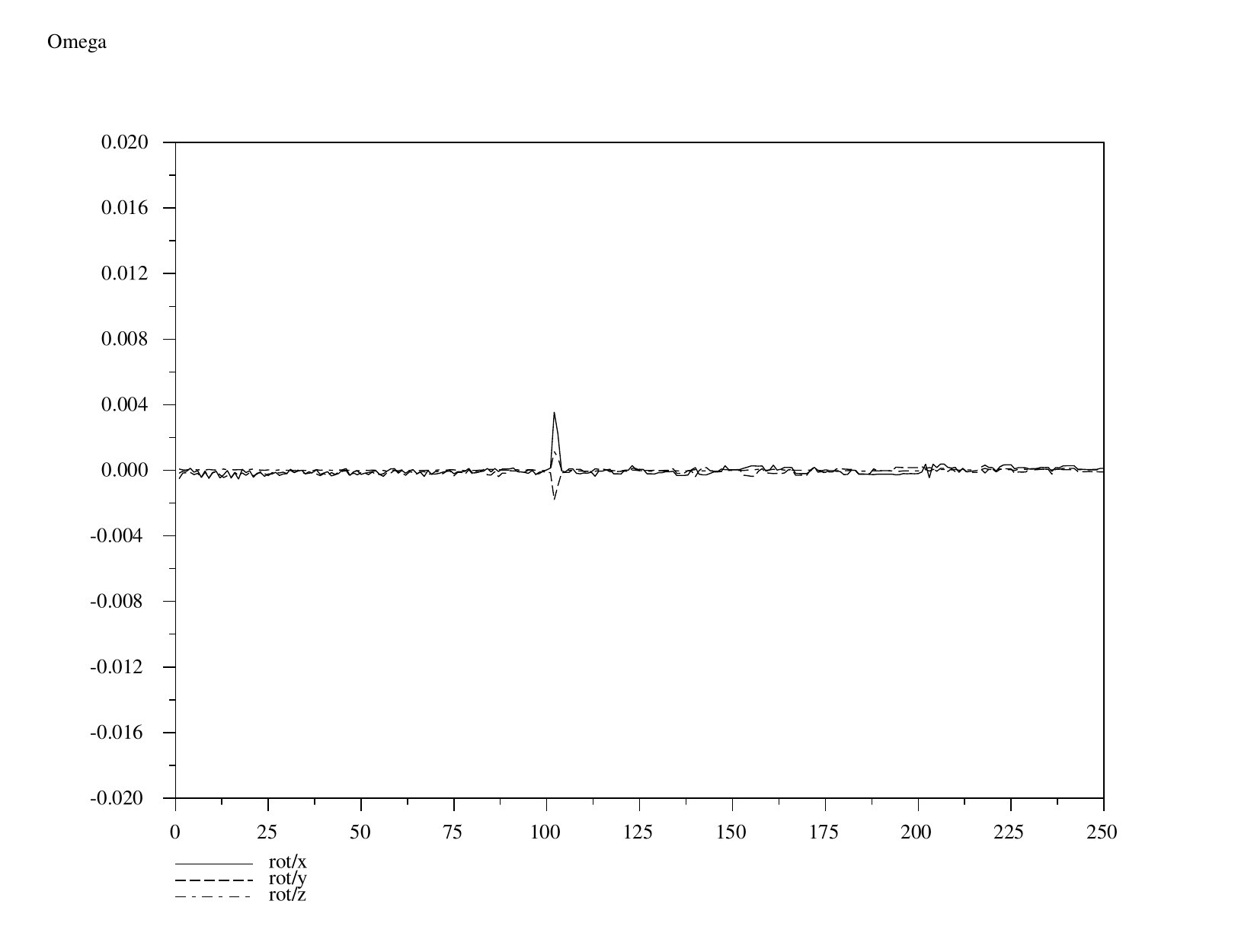}}
\rotatebox{-90}{\includegraphics[width=0.21\textwidth]{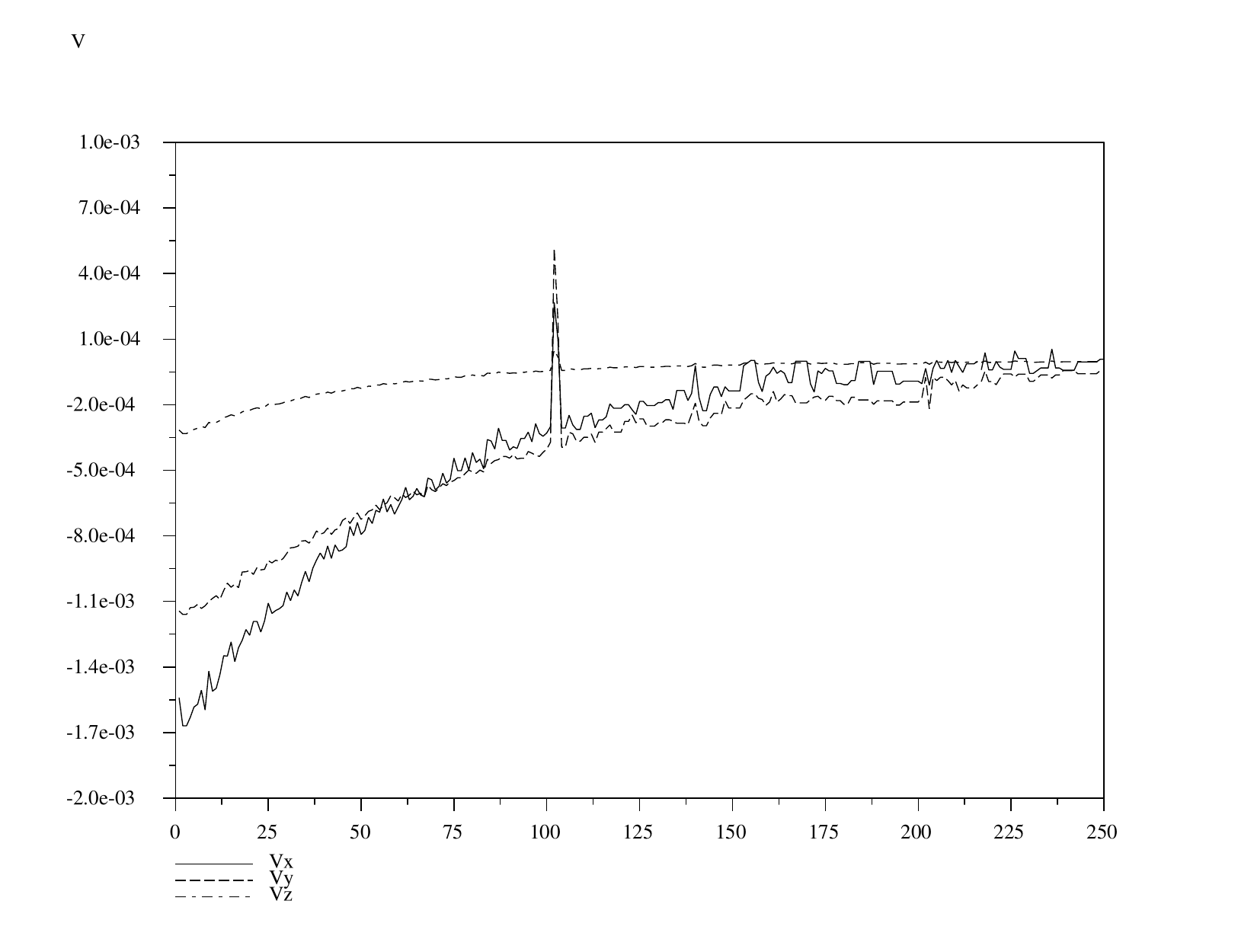}}
}
\caption{\label{transl_tau} Pure translation:  Rotation (left) and
translation (right) control.}
\end{figure}
\begin{figure}[htb]
\psfrag{h}{$\sum \|\hn_i - \hn_i^* \|$}
\psfrag{u}{$\sum \| \unit{u}_i \times \unit{u}_i^* \|$}
\centerline{%
\rotatebox{-90}{\includegraphics[width=0.21\textwidth]{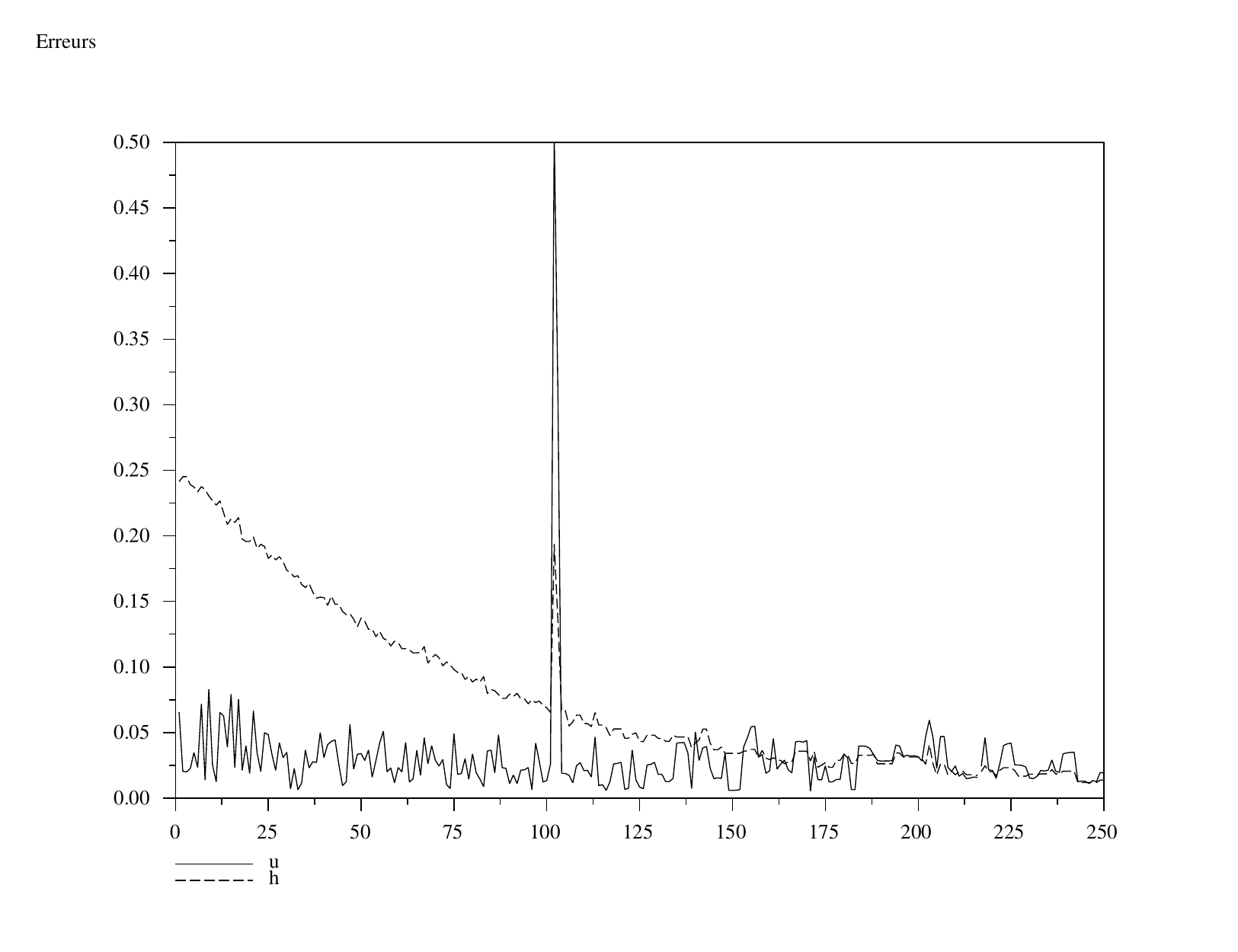}}
\rotatebox{-90}{\includegraphics[width=0.21\textwidth]{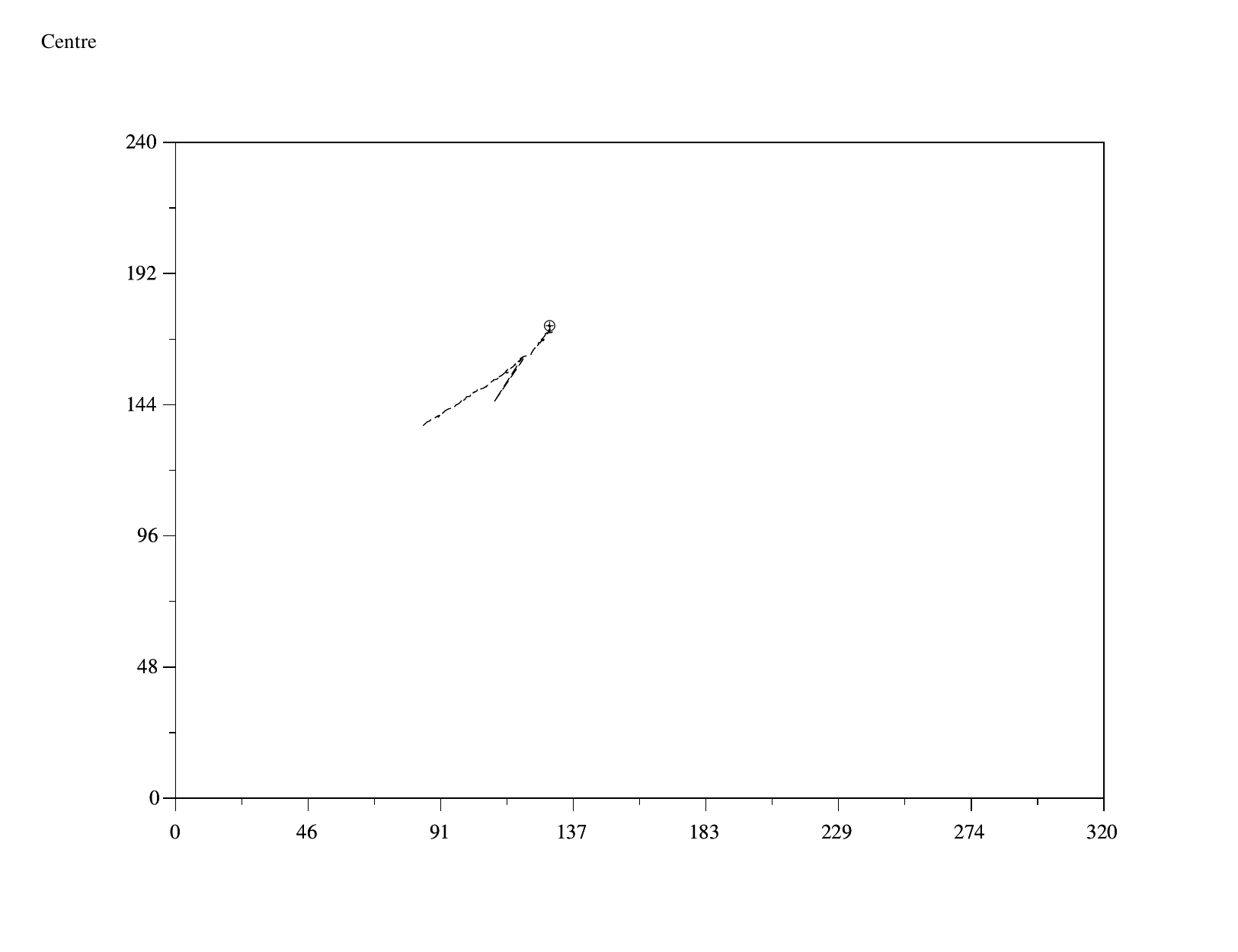}}
}
\caption{\label{transl_err} Pure translation: Errors (left) and
image trajectory of the trihedron center (right).}
\end{figure}

\paragraph{Pure translation --}
In the first experiment, we took the camera some ten centimeters away
from its goal position
(Figure~\ref{consigne_trans}). 
Notice first that during the realization of the task, a slight
tracking error occured (Figure~\ref{transl_uh}, left), which yields
perturbations on the following curves. Notice also that the 
computation of the \ddd line orientations is stable since the tracking
error does not affect it much (Figure~\ref{transl_uh},
right). Moreover, these orientations are constant, yielding a zero
control in orientation (Figure~\ref{transl_tau}, right). Hence, the
controlled motion is a pure translation as expected. This exhibits the
partial decoupling of our control law between translation and
rotation. The convergence in translation is here exponential
(Figure~\ref{transl_tau}, left). As expected the error in the image
({\small $\sum_{i=1}^3 \|\hn_i - \hn_i^* \|$}) decreases and the
orientation error is zero up to noise (Figure~\ref{transl_err}, left). 
Finally, notice the almost straight image trajectory of the trihedron
center (Figure~\ref{transl_err}, right).

% Images rotation /Z
% \begin{figure}[p]
% \psfrag{hi}{}
% \psfrag{ui}{}
% \centerline{%
% \rotatebox{-90}{\resizebox{!}{7cm}{\includegraphics{rot_hi-eps-converted-to.pdf}}}
% \rotatebox{-90}{\resizebox{!}{7cm}{\includegraphics{rot_ui-eps-converted-to.pdf}}}
% }
% \caption{\label{rot_uh} Pure rotation:  coefficients of the vectors
% $\hn_i$ (left) and $\unit{u}_i$ (right) along the control.}
% \end{figure}
%
\begin{figure}[tb]
\psfrag{V}{ }
\psfrag{Omega}{ }
\centerline{%
\rotatebox{-90}{\includegraphics[width=0.21\textwidth]{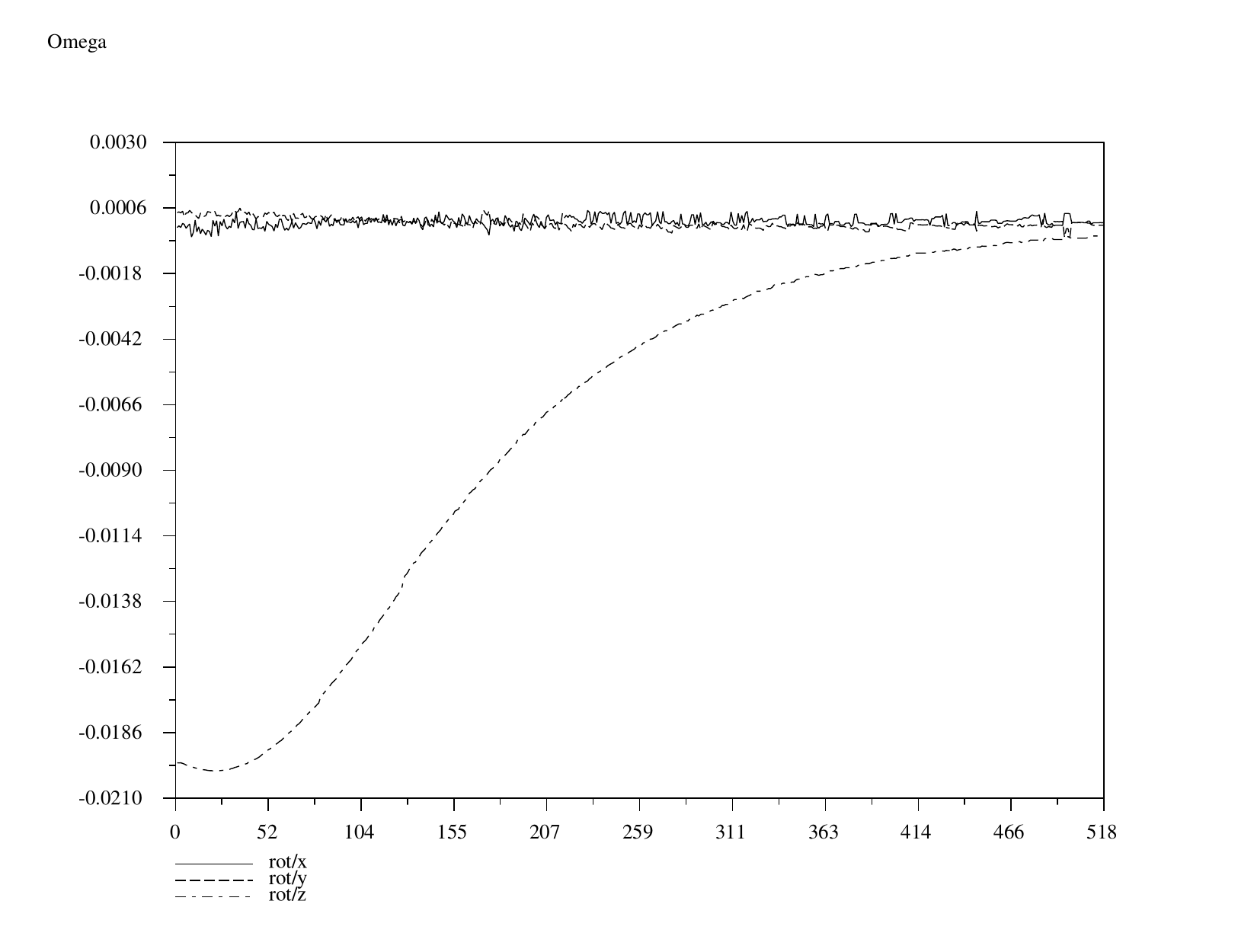}}
\rotatebox{-90}{\includegraphics[width=0.21\textwidth]{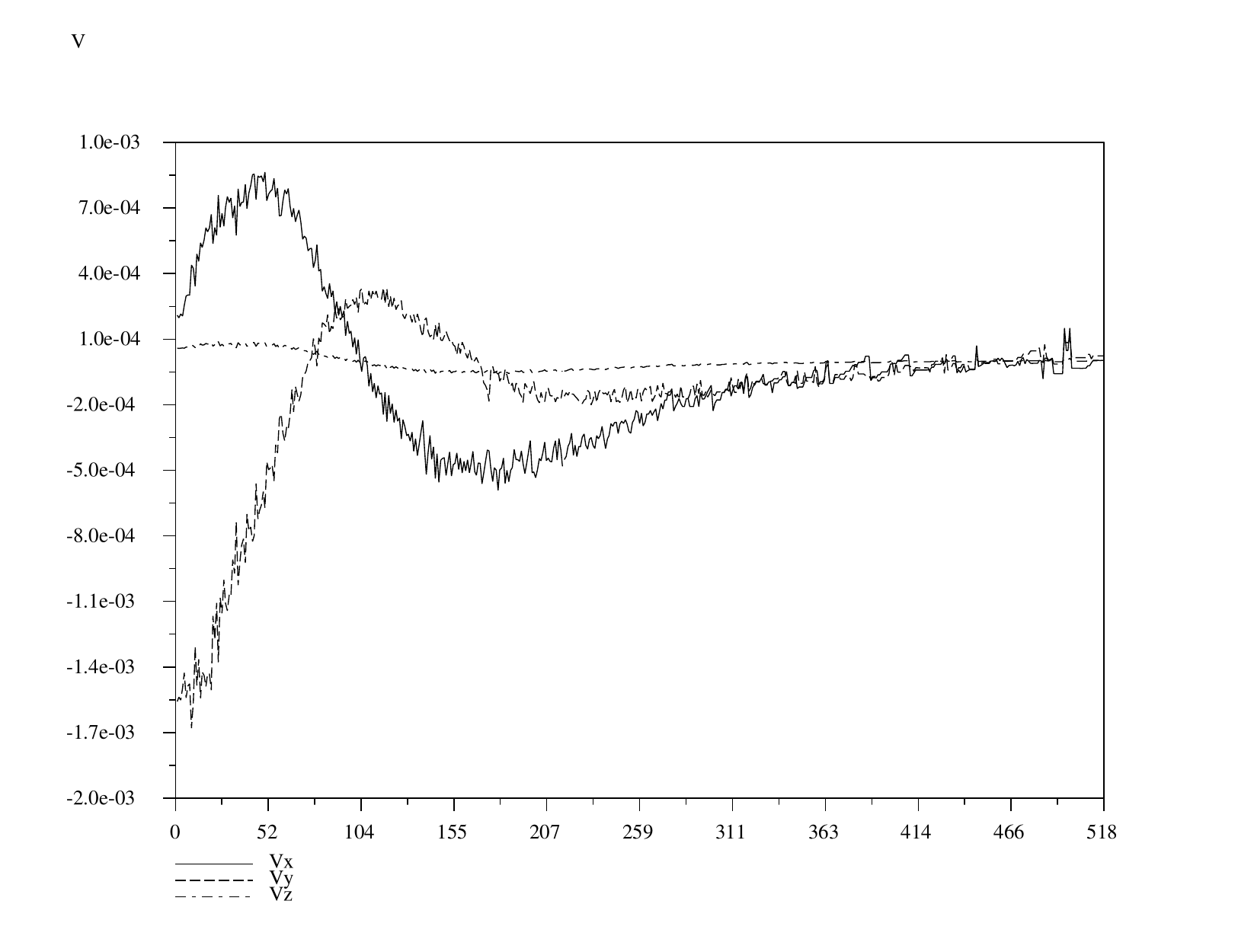}}
}
\caption{\label{rot_tau} Pure rotation:  Rotation (left) and
translation (right) control.}
\end{figure}
\begin{figure}[tb]
\psfrag{h}{$\sum \|\hn_i - \hn_i^* \|$}
\psfrag{u}{$\sum \| \unit{u}_i \times \unit{u}_i^* \|$}
\centerline{%
\rotatebox{-90}{\includegraphics[width=0.21\textwidth]{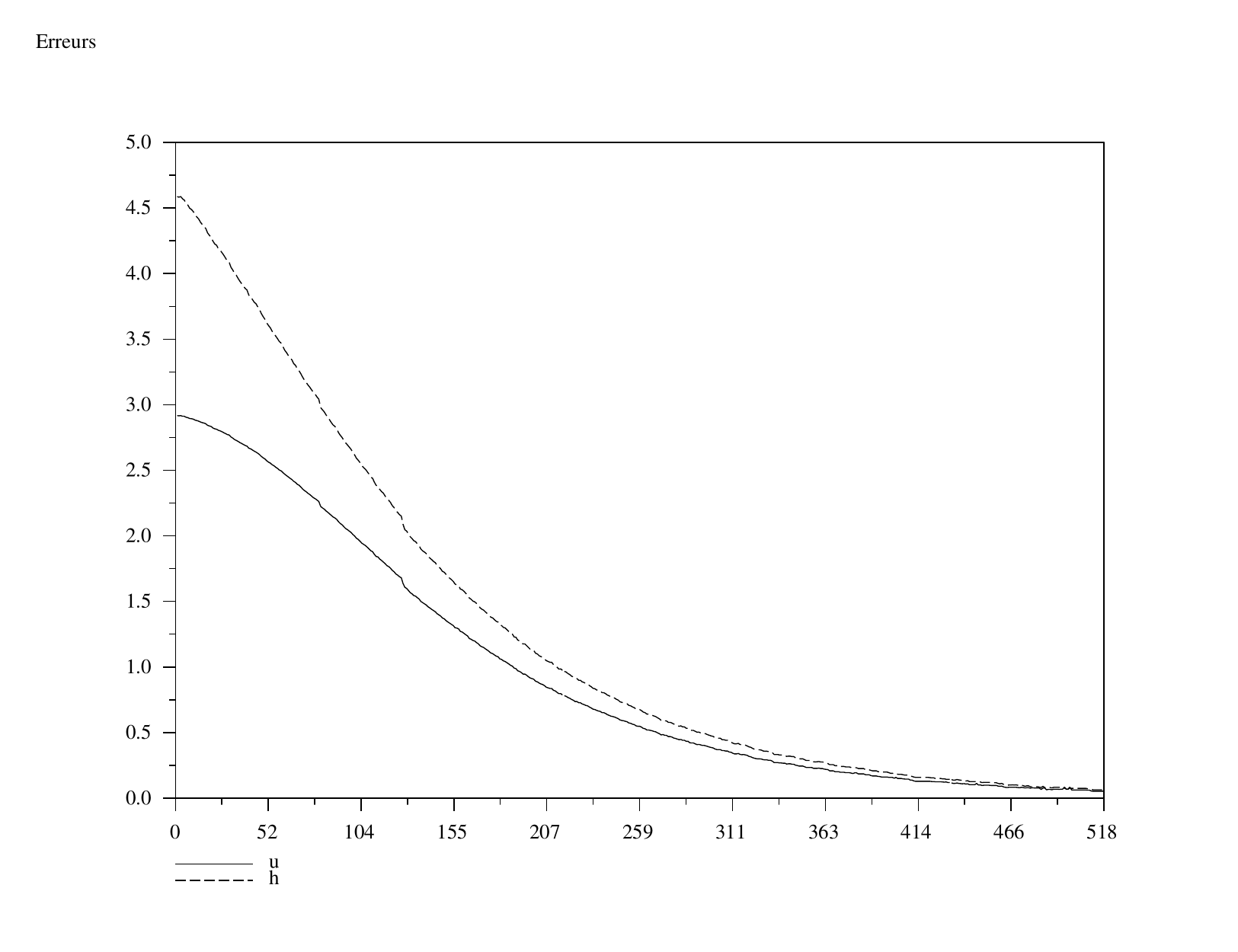}}
\rotatebox{-90}{\includegraphics[width=0.21\textwidth]{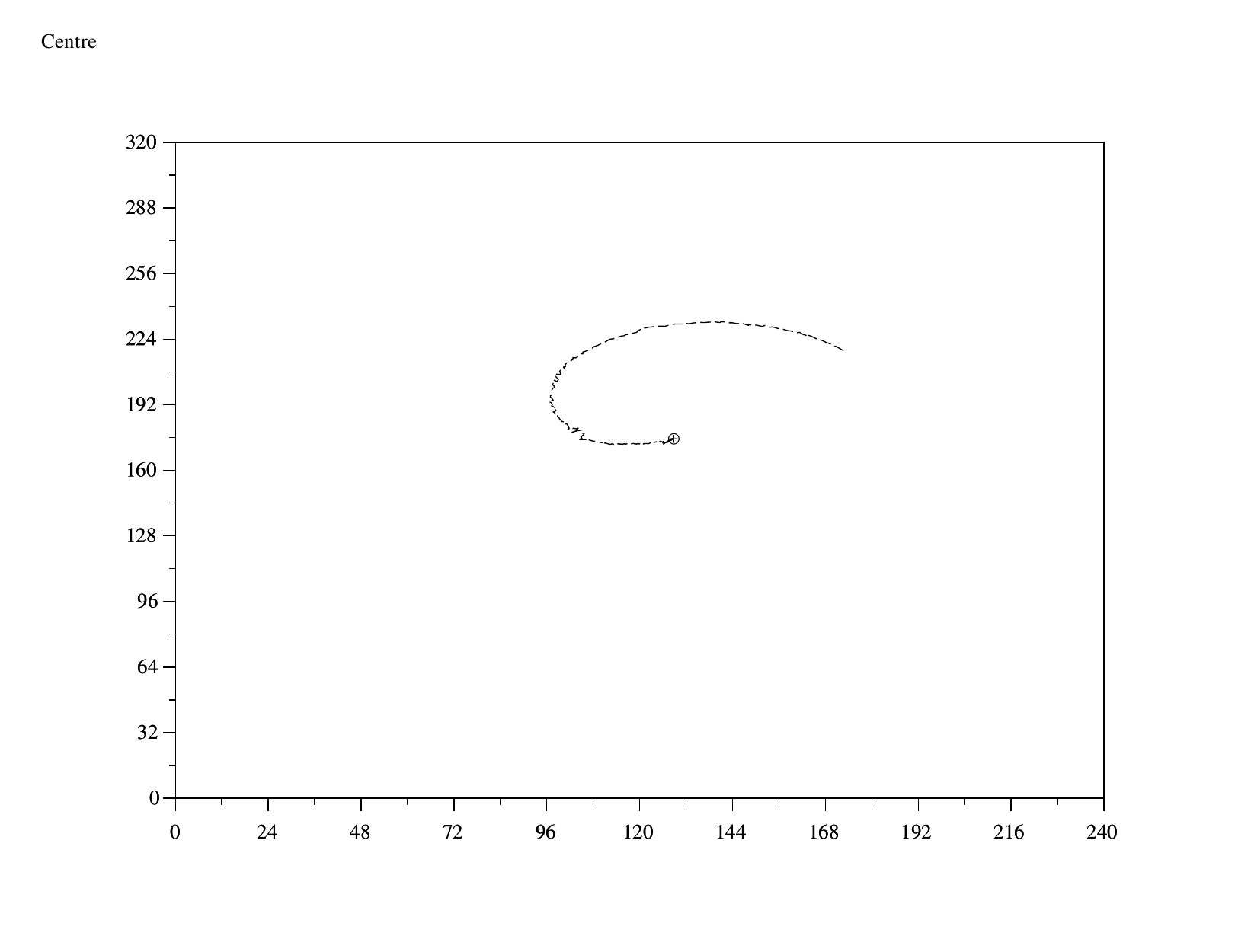}}
}
\caption{\label{rot_err} Pure rotation: Errors (left) and
image trajectory of the trihedron center (right).}
\end{figure}

\begin{figure*}[tb]
\centerline{
\resizebox{!}{4cm}{\includegraphics{consigne.pdf}}
\resizebox{!}{4cm}{\includegraphics{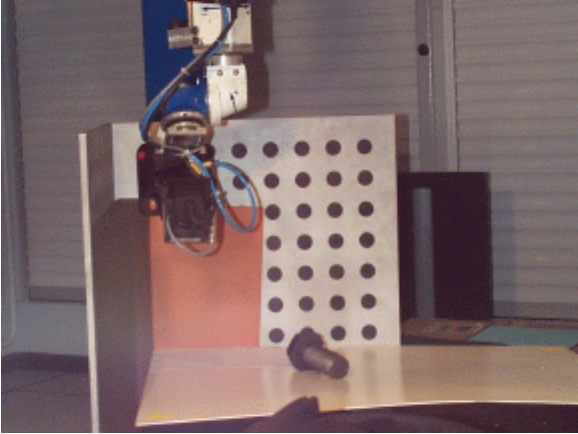}}
%\rotatebox{180}{\resizebox{!}{3cm}{\includegraphics{image10.pdf}}}
\resizebox{!}{4cm}{\includegraphics{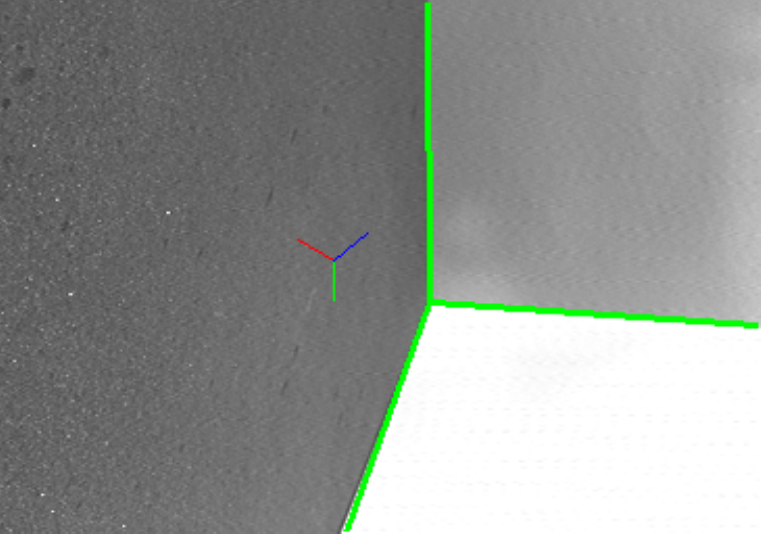}}
}
\caption{\label{consigne_cplx3} Complex motion. Left: robot desired position. Center:
robot initial position. Right: initial image (superimposed: desired
image of the trihedron). 
}  
\end{figure*}

\begin{figure}[tb]
\psfrag{V}{ }
\psfrag{Omega}{ }
\centerline{%
\rotatebox{-90}{\includegraphics[width=0.21\textwidth]{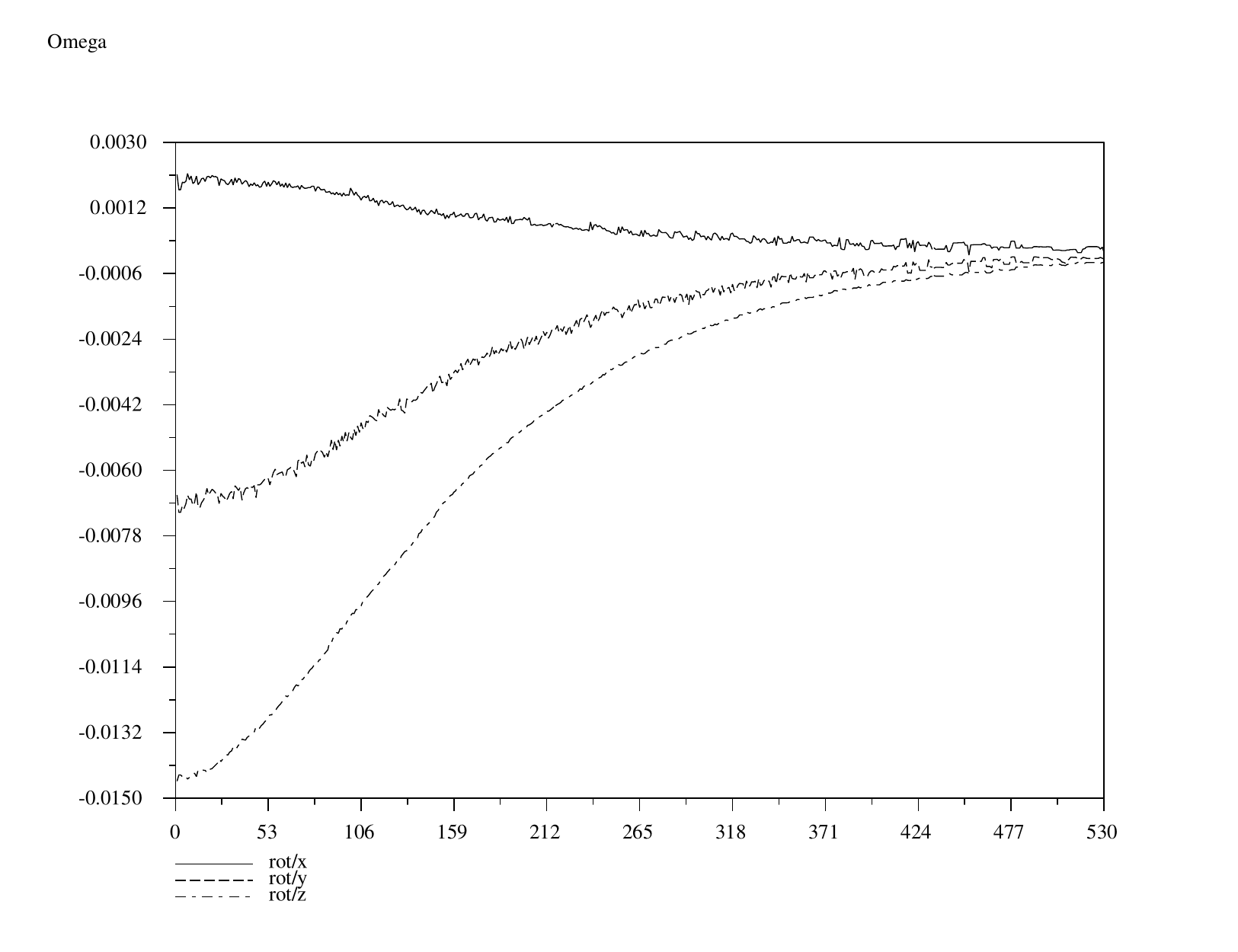}}
\rotatebox{-90}{\includegraphics[width=0.21\textwidth]{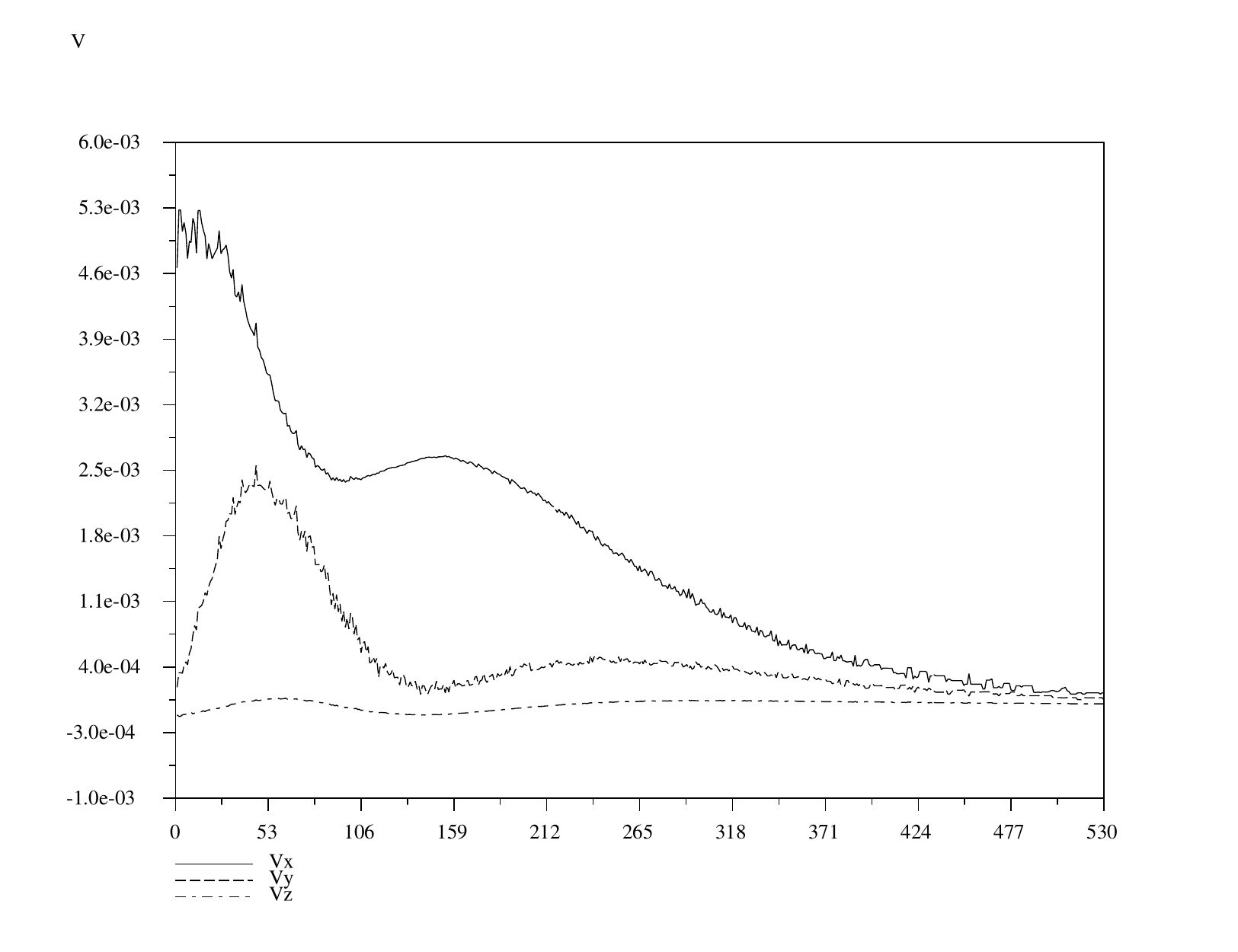}}
}
\caption{\label{cplx3_tau} Complex motion: Rotation (left) and
translation (right) control.}
\end{figure}
\begin{figure}[tb]
\psfrag{h}{$\sum \|\hn_i - \hn_i^* \|$}
\psfrag{u}{$\sum \| \unit{u}_i \times \unit{u}_i^* \|$}
\centerline{%
\rotatebox{-90}{\includegraphics[width=0.21\textwidth]{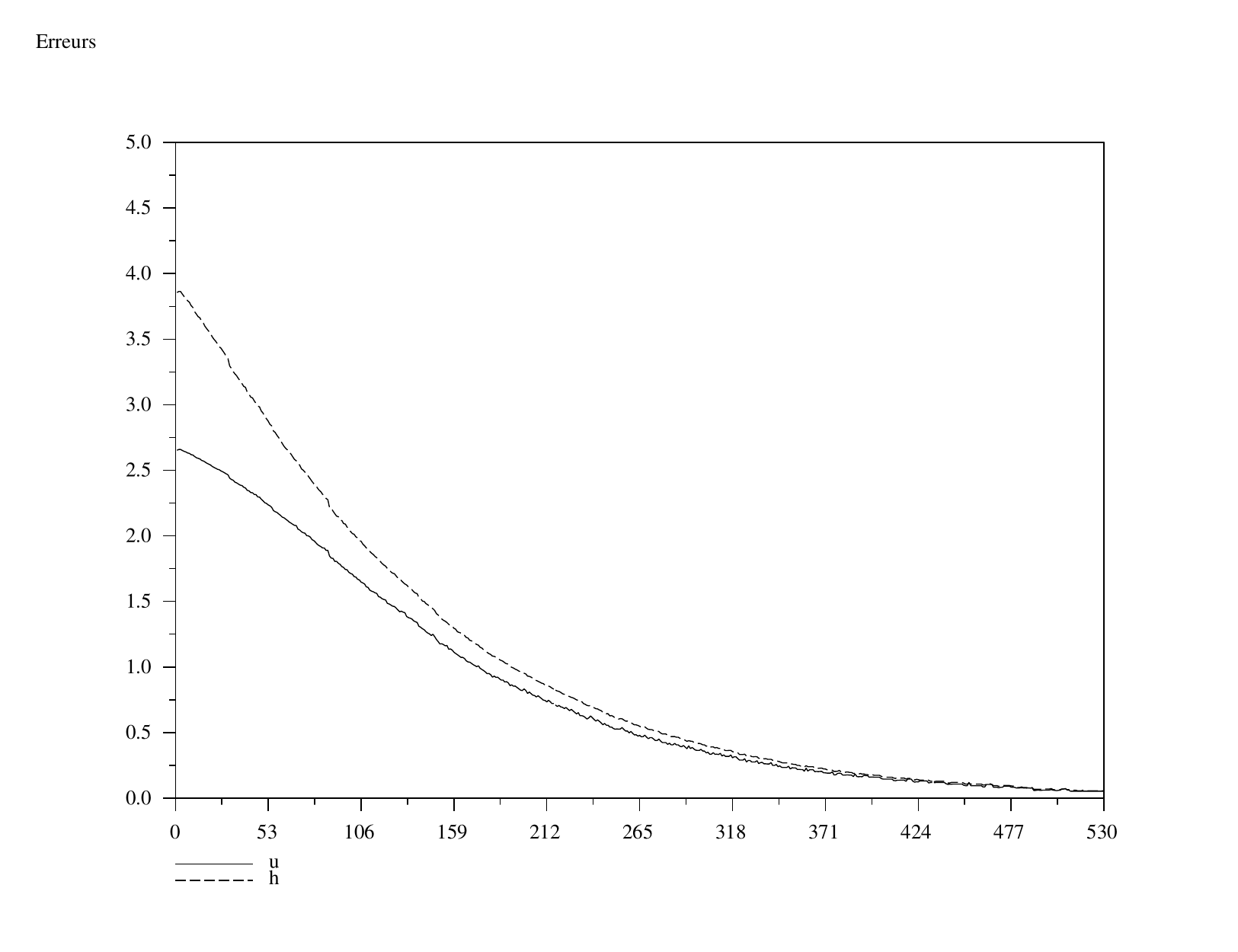}}
\rotatebox{-90}{\includegraphics[width=0.21\textwidth]{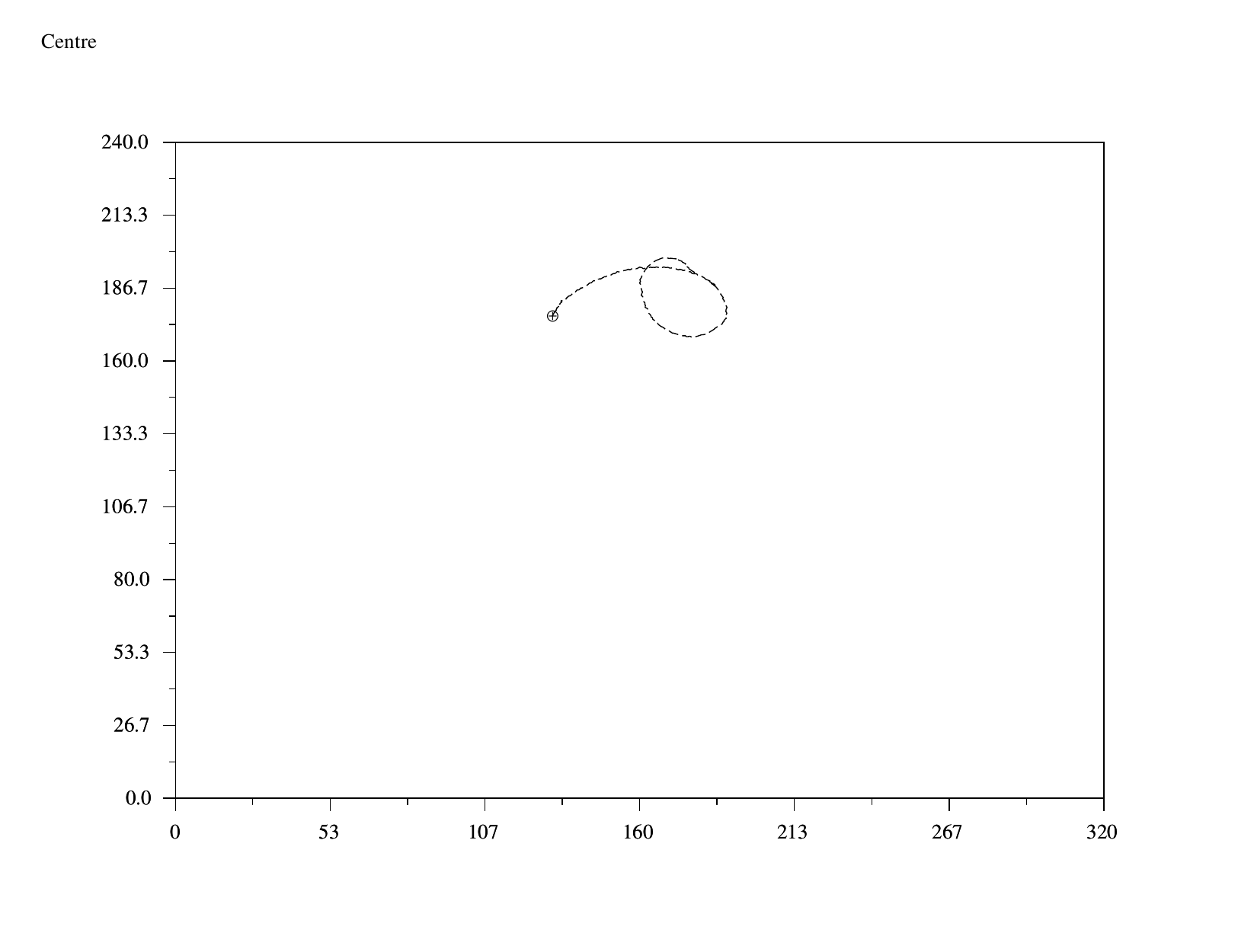}}
}
\caption{\label{cplx3_err} Complex motion: Errors (left) and
image trajectory of the trihedron center (right).}
\end{figure}

\paragraph{Pure rotation --}
In the same experimental conditions, we rotated the camera 100deg
along its optical axis away from its desired position.
One can notice that the orientation control operates, as expected,
solely along the optical axis (Figure~\ref{rot_tau}, left) and 
yields an asymptotically stable translation control
(Figure~\ref{rot_tau}, right), which keeps the trihedron visible
(Figure~\ref{rot_err}, right).

It is noticeable that the rotation control does not immediately
converge exponentially but starts with a slight acceleration
(Figure~\ref{rot_tau}, left). Nevertheless, both orientation error and
error in the image decrease (Figure~\ref{rot_err}, left).
Indeed, the residual angle between the initial and desired orientations is
larger than $\pi/2$. Hence,  the orientation control, which is of the
form $\un \times \un^*$ and thus proportionnal to the sine of the
residual angle, will increase in norm until the residual angle reaches
$\pi/2$ and then will decrease exponentially.

\paragraph{Complex motion --}
We end up this experimental section with a complex motion, in the
sense that it is composed of a translation of some 50cm and a
rotation of about 80deg  (Figure~\ref{consigne_cplx3}).
Since the initial camera position
is close to one the trihedron edges, it is challenging both for the 
tracking and the control algorithms. Indeed, the tracking algorithm
must avoid confusing the two image lines that are close to each other.
As for the control point of view, the difficulty comes from the fact
that this initial position is close to a singularity (\S~\ref{triedre}).

%
%\begin{figure}[hbtp]
%\centerline{
%\caption{\label{pos_init_bord} Une image initiale
%difficile. N.B.~: La consigne dans l'image est représentée par le petit
%trièdre tricolore.
%}
%\end{figure}

The orientation control is now again strictly decreasing
(Figure~\ref{cplx3_tau}, left) since the
residual angle is lower than $\pi/2$. On the opposite, the translation
control is disturbed by the orientation control
(Figure~\ref{cplx3_tau}, right) but still converges. Both orientation
errors and errors in the image are strictly decreasing
(Figure~\ref{cplx3_err}, left). As for
the image trajectory of the trihedron center, it is not as nice as in
the previous cases (Figure~\ref{cplx3_err},  right) and depends on the
relative choice of the control gains, $\mu_u$ and $\mu_h$. Indeed, this
trajectory was obtained with $\mu_u = 0.008$ and $\mu_h=0.01$ 
while the choice  $\mu_u=\mu_h=0.01$ did not keep the trihedron
visible. 

\section{Conclusion and Perspectives}
In this work, we achieved a theoretical and geometrical study
of several line representations in order to
determine the one which is best suited to visual servoing.
This led us to define a new representation, the so-called
binormalized Pl\"ucker coordinates, which has two advantages from the
control point of view. The first one is that it allowed us to define
properly the concept of image line alignment, which gathers both 2D and
3D information. The second one is that we could define a control law
which realizes such an  image line alignment with some very
interesting properties. First, the proposed law is an analytic
inversion of the motion equations and does not need any numerical
inversion of a Jacobian. Second, it shows a partial decoupling between
translation and rotation. Third, global asymptotic stability
conditions could be found, with a geometrical interpretation in the
particular case of the orthogonal trihedron.
This control law was validated both in simulation and experimentally.

However, there still remain several issues to be developed.
In a first time, stability conditions should be extended to the case
where orientation and translation controls are activated
simultaneously. This may be eased by recent results in the field of
cascaded systems\citep{panteley98}. In parallel, one could think of
studying the robustness of the proposed control law to measurement
and/or calibration errors. Another point to be adressed consists of
the line orientation computation in the general case, one simple way
being to extend our result to the stereoscopic case.
As for the practical side, we still need to reduce drastically the cycle
time of the control loop.

\appendix

\section{Proofs}

\subsection{Theorem~\ref{theoreme_rotation}}
\label{preuve_theo_rot}

Firstly, let us state the following lemma:
{\lemma
\label{lemme_inter}
For all configurations of rigidly linked lines:
\begin{equation}
 \sum_{i=1}^n \un_i \times \un_i^* = 0
\Rightarrow
 \unit{u}_i = \pm \unit{u}_i^*, i=1..n
\end{equation}
}
\\
{\bf Proof :} Recall the Rodrigues formula, which gives the image
$\vect{v}'$ of a vector $\vect{v}$ by a rotation of axis $\unit{n}$
and angle $\theta$:
\begin{equation}
\label{rodrigues}
\vect{v}' = \vect{v} + \sin \theta ( \unit{n} \times \vect{v} ) + ( 1
- \cos \theta ) \unit{n} \times ( \unit{n} \times \vect{v} )
\end{equation}

Since we consider a set of rigidly linked lines, the $\unit{u}_i$'s
differ from the $\unit{u}_i^*$'s by a unique rotation. Hence, we can
apply~(\ref{rodrigues}) for all $i=1..n$:
\begin{equation}
\unit{u}_i = \unit{u}_i^* + \sin \theta ( \unit{n} \times \unit{u}_i^* ) + ( 1
- \cos \theta ) \unit{n} \times ( \unit{n} \times \unit{u}_i^* ),
\quad i=1..n
\end{equation}

Building for all $i=1..n$, the crossproduct $\unit{u}_i \times
\unit{u}_i^*$ and expressing $\unit{u}_i^*$ on the orthogonal
basis $(\unit{n},  \unit{n} \times \unit{u}_i^* , \unit{n} \times (
\unit{n} \times \unit{u}_i^* ) )$ yields:
\begin{eqnarray*}
\unit{u}_i \times \unit{u}_i^* &=& \sin \theta \Big[ (\unit{n}^T
\unit{u}_i^*)^2 - 1\Big] \unit{n} \\
&+& \Big[ (1- \cos \theta) \unit{n}^T
\unit{u}_i^* + \alpha_i \unit{n}^T \unit{u}_i^* \Big]
\underbrace{( \unit{n} \times \unit{u}_i^* )}_{\perp \unit{n}} \\
&+& 
\Big[ \beta_i  \unit{n}^T \unit{u}_i^* \Big] \underbrace{\unit{n} \times (
\unit{n} \times \unit{u}_i^* )}_{\perp \unit{n}}
\end{eqnarray*}
Consequently, if the sum of these crossproducts is zero, then
\begin{equation}
\sum_{i=1}^n \sin \theta \Big[ (\unit{n}^T \unit{u}_i^*)^2 - 1\Big]=0
\end{equation}
From this equation, we deduce immediately that either $\sin \theta =0$
or $\unit{u}_i^* = \pm \unit{n}, \forall i=1..n$. Discarding the
second case which can be shown to be a degenerate case of the first
one, we obtain the final result as an interpretation of the first
case. \hfill $\Box$

Now, let us prove Theorem~\ref{theoreme_rotation}.

Consider the Lyapunov function $L=\frac{1}{2} \sum_{i=1}^n \|
\unit{u}_i - \unit{u}_i^*\|^2$.
Its derivative is:
\begin{equation}
\label{dLu}
\dot{L}_u = - \sum_{j=1}^{n} \dot{\unit{u}}_j^T  \unit{u}^*_j 
\end{equation}
From the dynamical equation~(\ref{dui}) and the control
definition~(\ref{rotation_control}), we can write~(\ref{dLu}) as:
\begin{equation}
\dot{L}_u = - \mu_u \sum_{j=1}^n \left( \left(\sum_{i=1}^n \un_i
\times \un_i^*\right) \times \un_j \right)^T \unit{u}_j^*
\end{equation}
Using the following property of the mixt product:
\begin{equation}
(\vect{u} \times \vect{v})^T \vect{w} = (\vect{v} \times \vect{w})^T \vect{u}
\end{equation}
we obtain:
\begin{equation}
\dot{L}_u = - \mu_u \left(\sum_{j=1}^n \un_j \times \unit{u}_j^*
\right)^T  \left(\sum_{i=1}^n \un_i
\times \un_i^*\right)
\end{equation}

Consequently, $\dot{L}$ is strictly negative unless $\sum_{i=1}^n \un_i
\times \un_i^*=0$. From Lemma~\ref{lemme_inter}, this means that
$\dot{L}$ is strictly negative unless $\un_i \neq \pm \un_i^*$. We
find here again the same unstable initial stationnary point as in the
single line case. 
Thus, the control~(\ref{rotation_control}) is asymptotically stable
provided that $\un_i(t=0) \neq - \un_i^*$, $\forall i=1..n$.
\hfill $\Box$

\subsection{Theorem~\ref{theoreme_translation}}
\label{preuve_theo_trans}

Consider the Lyapunov function defined by:
\begin{equation}
L_h = \frac{1}{2}\sum_{j=1}^n \| \hn_j - \hn_j^*\|^2
\end{equation}
the derivative of which is:
\begin{equation}
\dot{L}_h = - \sum_{j=1}^n \dot{\hn}_j^T \hn_j^*
\end{equation}

From (\ref{dhni}), where $\Vrot=0$, we have:
\begin{equation}
\dot{L}_h = - \sum_{j=1}^n \left(-\frac{\Vlin^T \hn_j}{h_j}(\un_j \times
\hn_j)\right)^T 
\hn_j^* 
\end{equation}
which rewrites:
\begin{equation}
\label{lyapunov}
\dot{L}_h =\sum_{j=1}^n\frac{\Vlin^T \hn_j}{h_j} \epsilon_j
\end{equation}
Inserting the control (\ref{cde_transl}) yields:
\begin{equation}
\dot{L}_h = - \mu_h \sum_{j=1}^n\frac{\left( \sum_{i=1}^n \matx{B}_i \epsilon_i
\hn_i \right)^T \hn_j}{h_j} \epsilon_j
\end{equation}
%
%Noting $\matx{W}=\left( \begin{smallmatrix}
%1/h_1&&\\
%&\ddots&\\
%&&1/h_n
%\end{smallmatrix} \right)$,
%we finally obtain:
%\begin{equation}
%\dot{L}_h = -  \mu_h \left( \begin{smallmatrix}
%\epsilon_1 \hn_1\\ \vdots \\ \epsilon_n \hn_n 
%\end{smallmatrix} \right)^T \left( \begin{matrix}
%\matx{B}_1, \hdots, \matx{B}_n 	
%\end{matrix} \right)^T \left( \begin{matrix}
%\id{3}, \hdots, \id{3} 
%\end{matrix} \right)
%\matx{W}\left( \begin{smallmatrix}
%\epsilon_1 \hn_1\\ \vdots \\ \epsilon_n \hn_n 
%\end{smallmatrix} \right)
%\end{equation}
%
This expression also rewrites under the following form which gives
the analytic condition for stationnary points to exist (i.e. when $\dot{L}_h
=0$):
\begin{equation}
\dot{L}_h = -  \mu_h \left( \sum_{i=1}^n \matx{B}_i \epsilon_i \hn_i
\right)^T
\left(  \sum_{i=1}^n \frac{ \epsilon_i  \hn_i }{h_i}\right)
\end{equation}

\hfill $\Box$

\subsection{Theorem~\ref{control_trihedron}}
\label{control_trihedron_proof}
Consider the same Lyapunov function $L$ as in the previous proof.
Then its derivative is given by~(\ref{lyapunov}) independently
from the translation control (provided that rotation control has converged). 
Inserting the control~(\ref{simple_control})
in~(\ref{lyapunov}) gives
\begin{equation}
\dot{L}_h = - \mu_h \sum_{j=1}^3\frac{\left( \sum_{i=1}^3 \epsilon_i
\hn_i \right)^T \hn_j}{h_j} \epsilon_j
\end{equation}
Noting $\matx{P}_j=\id{3}-\hn_j\hn_j^T$, the projection operator orthogonal
to $\hn_j$, and $\vect{e}_j=\hn_j-\hn_j^*$, we remark that:
\begin{equation}
\epsilon_j \hn_j = \un_j \times (\matx{P}_j \vect{e}_j ), \ \forall j
\end{equation}
Then, using this equation, we can rewrite the expression of $\dot{L}$ as
\begin{align}
\dot{L}_h &= - \mu_h \sum_{j=1}^3\frac{1}{h_j}
\left( \sum_{i=1}^3  \un_i \times(\matx{P}_i \vect{e}_i )\right)^T
\left(\un_j \times (\matx{P}_j \vect{e}_j )\right) \\
&=- \mu_h \sum_{j=1}^3\frac{1}{h_j}
\sum_{i=1}^3\left(\un_i \times(\matx{P}_i \vect{e}_i )\right)^T
\left(\un_j \times (\matx{P}_j \vect{e}_j )\right) \\
&=- \mu_h \sum_{j=1}^3\frac{1}{h_j}\sum_{i=1}^3
\left[\left(\un_i \times(\matx{P}_i \vect{e}_i )\right) \times \un_j\right]^T
(\matx{P}_j \vect{e}_j ) \\
&=- \mu_h \sum_{j=1}^3\frac{1}{h_j}\sum_{i=1}^3
\left[(\un_i^T\un_j)(\matx{P}_i \vect{e}_i)-\un_j^T(\matx{P}_i
\vect{e}_i)\un_i\right]^T (\matx{P}_j \vect{e}_j )
\end{align}
which finally gives:
\begin{equation}
\dot{L}_h = - \mu_h \sum_{j=1}^3\frac{1}{h_j}
\sum_{i=1}^3 \left[ (\un_i^T\un_j-\un_i\un_j^T)(\matx{P}_i \vect{e}_i)
\right]^T (\matx{P}_j \vect{e}_j )
\end{equation}

Define now
\[
E = \begin{pmatrix}
    \matx{P}_1 \vect{e}_1\\  \matx{P}_2\vect{e}_2\\  \matx{P}_3\vect{e}_3
    \end{pmatrix}, \quad
U =  \begin{pmatrix}
    \unit{u}_1\\ \unit{u}_2\\ \unit{u}_3
    \end{pmatrix}
\]
and the operator:
\begin{align*}
\matx{P}_U  & = \frac{1}{2}(\id{9} - U \otimes U^T) \\
& = 
\begin{pmatrix}
\id{3} - \unit{u}_1 \unit{u}_1^T&- \unit{u}_2 \unit{u}_1^T&-
\unit{u}_3 \unit{u}_1^T\\
- \unit{u}_1 \unit{u}_2^T&\id{3} - \unit{u}_2\unit{u}_2^T&-
\unit{u}_3 \unit{u}_2^T\\
- \unit{u}_1 \unit{u}_3^T& - \unit{u}_2\unit{u}_3^T& \id{3} -
\unit{u}_3 \unit{u}_3^T
\end{pmatrix}
\end{align*}
One can easily prove that $\matx{P}_U$ is an orthogonal projection onto the
kernel of $U^T$, hence on the hyperplane $\Pi$ orthogonal to $U$.

After convergence of the rotational control, we have
\[
	\unit{u}_j = \unit{u}_j^*, \quad j=1..3
\]
from which we deduce that
\begin{equation}
\label{eTu}
\unit{u}_j^T (\matx{P}_j \vect{e}_j) = 0, \quad j=1..3
\end{equation}

Using this result and the trihedron orthogonality, we obtain
\begin{equation}
\dot{L} = - 2 \mu_h \tilde{E}^T \matx{P}_U E
\end{equation}
where $\tilde{E}$ is the vector obtained by concatenating the
$\tilde{\vect{e}}_j = \frac{1}{h_j}\matx{P}_j \vect{e}_j$ for all $j=1..3$.

From (\ref{eTu}), we find that $E$ and $U$ are orthogonal, hence
$E \in \Pi$ and $\matx{P}_U E = E$. This simplifies the expression
of $\dot{L}$ into:
\[
\dot{L} = -2 \mu_h \tilde{E}^T E = -2 \mu_h \sum_{j=1}^3 \frac{1}{h_j} \vect{e}'{}_j^T \vect{e}'_j
\]
where $\vect{e}'_j = \matx{P}_j \vect{e}_j$.

Consequently, this sum is negative and only vanishes when all $\vect{e}'_j$ 
are zero. After the convergence of the rotation control, this is only
possible when $\hn_j = \pm \hn_j^*$.

\hfill $\Box$

%% SUPPR %% \section{Notation in Proposition~\ref{sdot}}
%% SUPPR %% \label{notations_prop_sdot}
%% SUPPR %% 
%% SUPPR %% \begin{figure}[tp]
%% SUPPR %% {\LARGE
%% SUPPR %% \psfrag{il}{$\vect{i}_l$} 
%% SUPPR %% \psfrag{jl}{$\vect{j}_l$}
%% SUPPR %% \psfrag{kl}{$\vect{k}_l$}
%% SUPPR %% \psfrag{z}{$\vect{z}_l$}
%% SUPPR %% \psfrag{a}{$\alpha_y$}
%% SUPPR %% \psfrag{g}{$\gamma$}
%% SUPPR %% \psfrag{th}{$\delta \theta_x$}
%% SUPPR %% \psfrag{h1}{$h_1$}
%% SUPPR %% \psfrag{h2}{$h_2$}
%% SUPPR %% \psfrag{h3}{$h_3$}
%% SUPPR %% \psfrag{n}{$\vect{n}$}
%% SUPPR %% \psfrag{dz}{$d\vect{z}_l$}
%% SUPPR %% \psfrag{dx}{$dx$}
%% SUPPR %% \centerline{\resizebox{!}{6cm}{
%% SUPPR %% \includegraphics{derives_pt_laser-eps-converted-to.pdf.gz}}}
%% SUPPR %% }
%% SUPPR %% \caption{\label{derivees} Variation du faisceau laser lors 
%% SUPPR %% d'une translation selon l'axe X (gauche) et d'une rotation autour de cet axe
%% SUPPR %% (droite)}
%% SUPPR %% \end{figure}

\section*{Acknowledgments}
This work was supported by the European 
  Community through the Esprit-IV reactive LTR project number 26247
(VIGOR). This work was completed while Nicolas Andreff was a Ph.D. student at INRIA
Rh\^one-Alpes.

\bibliographystyle{spbasic}
%\bibliography{avdroites,tracking}

\end{document}